\documentclass[lettersize,journal]{IEEEtran}

\usepackage[utf8]{inputenc}
\usepackage[T1]{fontenc}
\usepackage[compatibility=false]{caption}
\usepackage{hyperref}
\usepackage{url}
\usepackage{booktabs}
\usepackage{amsfonts}
\usepackage{nicefrac}
\usepackage{microtype}
\usepackage[table]{xcolor}
\usepackage{tabularx}
\usepackage{colortbl}
\usepackage{ulem}
\usepackage{amsthm}
\usepackage{amsmath}
\usepackage{array}
\usepackage[noend]{algcompatible}
\usepackage{algorithm}
\usepackage{graphicx}
\usepackage{multirow}
\usepackage{subcaption}
\usepackage{wrapfig}
\usepackage{enumitem}
\usepackage{makecell}
\usepackage{pifont}
\algrenewcommand{\algorithmiccomment}[1]{\quad// #1}

\newtheorem{theorem}{Theorem}

\newtheorem{proposition}{Proposition}

\newtheorem{remark}{Remark}

\PassOptionsToPackage{numbers}{natbib}

\definecolor{myblue2}{RGB}{68,114,196}
\definecolor{myblue}{RGB}{0,80,157}
\definecolor{citecolor}{RGB}{34, 149, 34}
\definecolor{pruningblue}{RGB}{36,143,255}

\definecolor{mygreen}{RGB}{0,100,0}

\DeclareMathOperator*{\argmin}{arg\,min}

\newcommand{\etal}[0]{\textit{et al.}}

\newcommand{\std}[2]{\begin{tabular}{@{}c@{}}#1{\color{gray}$\scriptscriptstyle  \pm  #2$}\end{tabular}}
\newcommand{\stdb}[2]{\begin{tabular}{@{}c@{}}\textbf{#1}{\color{gray}$\scriptscriptstyle  \pm  #2$}\end{tabular}}

\newcommand{\up}[2]{\begin{tabular}{@{}c@{}}#1{\color{mygreen}$\scriptscriptstyle  \uparrow  #2$}\end{tabular}}
\newcommand{\down}[2]{\begin{tabular}{@{}c@{}}#1{\color{myblue}$\scriptscriptstyle  \downarrow  #2$}\end{tabular}}

\definecolor{lightblue}{RGB}{216, 239, 248}
\definecolor{lightgreen}{RGB}{229, 245, 222}
\definecolor{darkred}{RGB}{186, 49, 50}
\definecolor{darkorange}{rgb}{1.0, 0.549, 0.0}
\definecolor{darkbrown}{rgb}{0.396, 0.263, 0.129}

\newcommand{\lpld}[0]{\cellcolor{lightblue!90}}
\newcommand{\lpqld}[0]{\cellcolor{lightgreen}}

\newcommand{\appendixnarrowtablewidth}{0.72\textwidth}
\newcommand{\appendixtablewidth}{0.80\textwidth}
\newcommand{\appendixwidetablewidth}{0.88\textwidth}

\begin{document}

\title{Soft Label Pruning and Quantization for Large-Scale Dataset Distillation}

\author{Lingao Xiao, Yang He
\IEEEcompsocitemizethanks{
\IEEEcompsocthanksitem Received 11 June 2025; revised 18 January 2026; accepted 8 February 2026. This research is supported by A*STAR Career Development Fund (CDF) under Grant C243512011, the National Research Foundation, Singapore under its National Large Language Models Funding Initiative (AISG Award No: AISG-NMLP-2024-003). Any opinions, findings and conclusions or recommendations expressed in this material are those of the author(s) and do not reflect the views of National Research Foundation, Singapore. Recommended for acceptance by J. Sun. (\textit{Corresponding author: Yang He.})
\IEEEcompsocthanksitem The authors are with the CFAR, Agency for Science, Technology and Research, Singapore 138632, also with the IHPC, Agency for Science, Technology and Research, Singapore 138632, and also with the National University of Singapore, Singapore 119077 (e-mail: xiao\_lingao@outlook.com; hyhy1992@gmail.com).
\IEEEcompsocthanksitem This article has supplementary downloadable material available at https://doi.org/10.1109/TPAMI.2026.3664488, provided by the authors.
\IEEEcompsocthanksitem Digital Object Identifier 10.1109/TPAMI.2026.3664488.}
}

\markboth{IEEE TRANSACTIONS ON PATTERN ANALYSIS AND MACHINE INTELLIGENCE, Feb 2026}%
{Soft Label Pruning and Quantization for Large-Scale Dataset Distillation}

\maketitle

\begin{abstract}
Large-scale dataset distillation requires storing auxiliary soft labels that can be 30-40$\times$ (ImageNet-1K) or 200$\times$ (ImageNet-21K) larger than the condensed images, undermining the goal of dataset compression. We identify two fundamental issues necessitating such extensive labels: (1) \textit{insufficient image diversity}, where high within-class similarity in synthetic images requires extensive augmentation, and (2) \textit{insufficient supervision diversity}, where limited variety in supervisory signals during training leads to performance degradation at high compression rates. To address these challenges, we propose Label Pruning and Quantization for Large-scale Distillation (LPQLD). We enhance \textit{image diversity} via class-wise batching and BN supervision during synthesis. For \textit{supervision diversity}, we introduce Label Pruning with Dynamic Knowledge Reuse to enhance \textit{label-per-augmentation diversity}, and Label Quantization with Calibrated Student-Teacher Alignment to enhance \textit{augmentation-per-image diversity}. Our approach reduces soft label storage by 78$\times$ on ImageNet-1K and 500$\times$ on ImageNet-21K while improving accuracy by up to 7.2\% and 2.8\%, respectively. 
Extensive experiments validate the superiority of LPQLD across different network architectures and other dataset distillation methods.
Code is available at \href{https://github.com/he-y/soft-label-pruning-quantization-for-dataset-distillation}{https://github.com/he-y/soft-label-pruning-quantization-for-dataset-distillation}.
\end{abstract}

\begin{IEEEkeywords}
Dataset distillation, dataset condensation, data compression, efficient learning.
\end{IEEEkeywords}
\section{Introduction}

\IEEEPARstart{W}{e} are advancing into the era of \textbf{ImageNet-level dataset condensation}, aiming to compress large datasets into much smaller, efficient counterparts~\cite{wang2018dataset, wang2022cafe, cazenavette2022distillation, zhao2021datasetGM, zhao2023dataset}.
While prior works struggled to scale due to memory constraints, recent methods~\cite{yin2023squeeze, yin2023dataset} employing a three-phase approach (squeeze, recover, relabel) have shown promise. However, a critical bottleneck emerges in the `relabel' phase: the auxiliary data, primarily \textbf{soft labels} generated under numerous augmentations across training epochs~\cite{yin2023squeeze, yin2023dataset, sun2023diversity, shao2023generalized, zhou2024self}, demand enormous storage as shown in Figure~\ref{fig:main-figure}. For ImageNet-1K, this auxiliary storage can be \textbf{30-40$\times$} larger than the condensed images themselves, and for ImageNet-21K, this ratio balloons to over \textbf{200$\times$}, severely undermining the goal of dataset compression and making large-scale distillation impractical.

We identify two fundamental issues that necessitate such extensive soft label storage. The first is \textbf{insufficient image diversity}, where synthesized images exhibit high within-class similarity (as observed in SRe$^2$L~\cite{yin2023dataset}). We attribute this to the common practice of constructing batches with samples from different classes to match global Batch Normalization (BN) statistics during image synthesis (the `recover' phase); this is confirmed by Feature Cosine Similarity and MMD (Sec.~\ref{sec: similarity-analysis}). 
The second is \textbf{insufficient supervision diversity}, where repeated supervision with the same soft labels across epochs limits the variety of supervisory signals and causes severe performance degradation (Fig.~\ref{fig: effect}). This is the core problem of LPLD~\cite{xiao2024lpld} at high compression rates. Notably, these two issues are interconnected: low image diversity drives heavier augmentation during student training, which in turn demands more diverse supervision to maintain performance.

\begin{figure}[t]
    \centering
    \includegraphics[width=\linewidth]{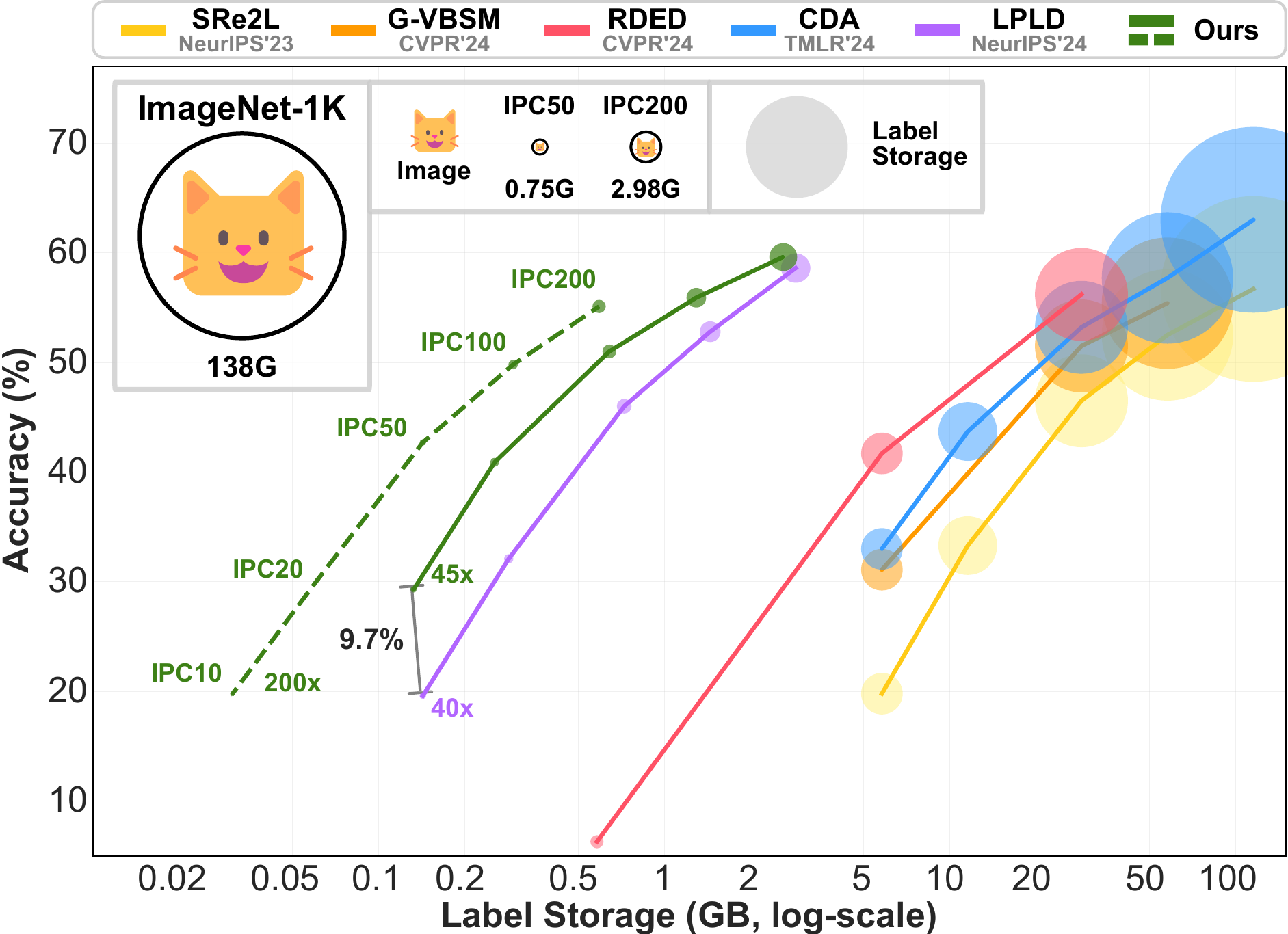}
    \caption{
    The relationship between performance and label storage needed for ImageNet-1K distillation. Our method creates a new Pareto front for the performance-storage trade-off, having \textbf{fewer soft labels than images}. Baselines~\cite{yin2023squeeze, yin2023dataset, shao2023generalized, sun2023diversity,xiao2024lpld} are shown for comparison.
    }
    \vspace{-0.4cm}
    \label{fig:main-figure}
\end{figure}

To address these dual challenges, we propose \textbf{Label Pruning and Quantization for Large-scale Distillation (LPQLD)}, a framework that systematically tackles both image diversity and supervision diversity issues.

\textbf{Addressing Image Diversity.} 
Following LPLD~\cite{xiao2024lpld}, we enhance image diversity during synthesis by batching images within the same class and introducing class-wise BN supervision, encouraging images to collectively capture diverse class characteristics. As illustrated in Fig.~\ref{fig: viz-ours} and quantified in Table~\ref{tab: cosine-similarity}, this approach yields synthetic images with lower within-class similarity and better representation of the original dataset (Fig.~\ref{fig:mmd}).

\textbf{Addressing Supervision Diversity.} 
We tackle supervision diversity from two complementary angles: label-per-augmentation diversity and augmentation-per-image diversity.

\textit{(a) Label-per-Augmentation Diversity via Dynamic Knowledge Reuse.}
At high compression rates with label pruning, each augmentation is repeatedly supervised by the same soft label across training epochs, severely limiting supervision diversity. To address this, we propose Dynamic Knowledge Reuse (DKR), a theoretically-grounded strategy employing temperature annealing to extract different supervisory signals from the same label across training epochs (Sec.~\ref{sec: dynamic-knowledge-reuse}). This enables each retained label to provide varied supervision throughout training without additional storage overhead.

\textit{(b) Augmentation-per-Image Diversity via Label Quantization with Calibrated Alignment.}
Relying solely on label pruning drastically reduces augmentations per image, potentially limiting each image to only one augmentation. We introduce Label Quantization (Q), which stores only top-k pre-softmax logits per label to enable more augmentation-label pairs under the same storage budget. However, quantization alters the teacher's probability distribution. Our Calibrated Student-Teacher Alignment (CA) addresses this by dynamically adjusting student distribution via grid-search optimization (Sec.~\ref{sec: student-teacher-alignment}).

\textit{(c) Balancing the Trade-off.}
With a fixed storage budget, increasing augmentation-per-image diversity (via quantization) necessarily reduces the number of labels available per augmentation (via pruning). We provide a comprehensive Pareto frontier analysis to identify optimal pruning-quantization configurations that balance these two types of supervision diversity for different dataset scales (Sec.~\ref{sec: additional-analysis}), offering practical guidance for deploying LPQLD across diverse scenarios.

\begin{figure}[t] 
    \begin{subfigure}{.48\columnwidth}
        \centering
        \includegraphics[width=\textwidth]{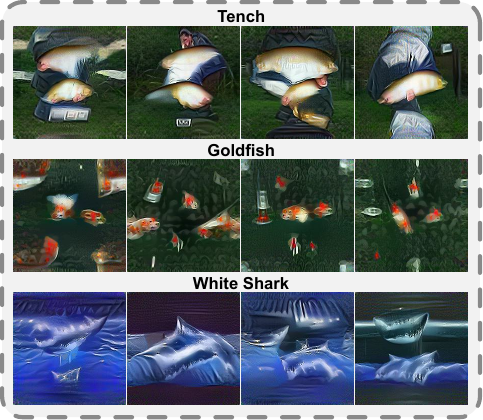}
        \caption{SRe$^2$L~\cite{yin2023squeeze}.}
        \label{fig: viz-sre2l}
    \end{subfigure}\hfill
    \begin{subfigure}{.48\columnwidth}
        \centering
        \includegraphics[width=\textwidth]{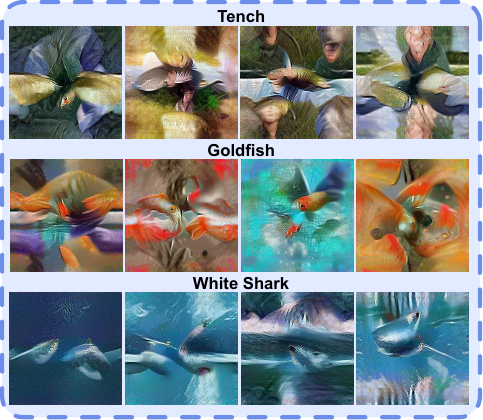}
        \caption{Ours (IPC=10).}
        \label{fig: viz-ours}
    \end{subfigure}
    \caption{Visual comparison of synthesized images (IPC=10) for ImageNet-1K classes. Our method exhibits noticeably higher \textbf{within-class visual diversity}.}
    \label{fig: viz}
\end{figure}

\begin{table}[t]
\centering
\caption{Comparison between LPLD and LPQLD.}
\setlength{\tabcolsep}{0.3em}
\label{tab:lpld-lpqld-comparison}
\scriptsize
\begin{tabular}{@{}p{0.18\columnwidth}p{0.55\columnwidth}cc@{}}
\toprule
\textbf{Category} & \textbf{Feature} & \textbf{LPLD} & \textbf{LPQLD} \\
\midrule
\begin{tabular}[t]{@{}l@{}}Image \\ Diversity\end{tabular} & \textbf{Class-wise Batching \& Supervision} \newline {\scriptsize Enhances within-class diversity} & \textcolor{mygreen}{\ding{51}} & \textcolor{mygreen}{\ding{51}} \\
\midrule
\multirow{2}{0.18\columnwidth}{Label Efficiency} & \textbf{Label Pruning} \newline {\scriptsize Reduces number of augmentation pairs} & \textcolor{mygreen}{\ding{51}} & \textcolor{mygreen}{\ding{51}} \\
\cmidrule(lr){2-4}
& \textbf{Label Quantization (Top-k)} \newline {\scriptsize Compresses individual label size} & \textcolor{red}{\ding{55}} & \textcolor{mygreen}{\ding{51}} \\
\midrule
\multirow{2}{0.18\columnwidth}{Supervision Diversity} & \textbf{Dynamic Knowledge Reuse (DKR)} \newline {\scriptsize Extracts diverse signals from pruned labels} & \textcolor{red}{\ding{55}} & \textcolor{mygreen}{\ding{51}} \\
\cmidrule(lr){2-4}
& \textbf{Calibrated Student-Teacher Alignment (CA)} \newline {\scriptsize Aligns student with quantized teacher} & \textcolor{red}{\ding{55}} & \textcolor{mygreen}{\ding{51}} \\
\bottomrule
\end{tabular}
\vspace{0.5em}
\end{table}

\begin{figure}[t] 
    \begin{subfigure}{.48\columnwidth}
        \centering
        \includegraphics[width=\textwidth]{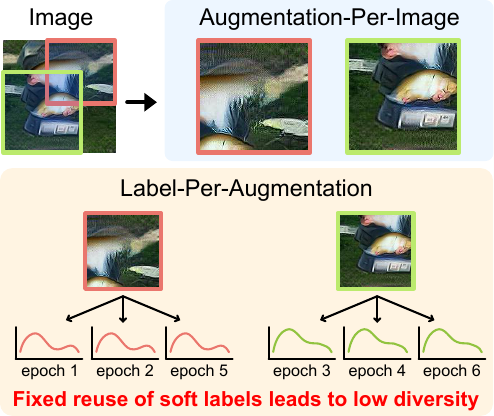}
        \caption{LPLD~\cite{xiao2024lpld}.}
        \label{fig: effect-lpld}
    \end{subfigure}\hfill
    \begin{subfigure}{.48\columnwidth}
        \centering
        \includegraphics[width=\textwidth]{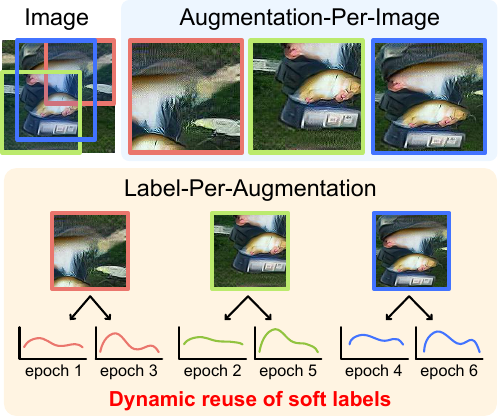}
        \caption{Ours.}
        \label{fig: effect-lpqld}
    \end{subfigure}
    \caption{Illustration of supervision diversity of LPLD and LPQLD (Ours) under the same storage and training budget.}
    \label{fig: effect}
    \vspace{0.5em} 
\end{figure}

\begin{table}[t]
\centering
\caption{Necessity of LPQLD Framework. P: Pruning, Q: Quantization, DKR: Dynamic Knowledge Reuse, CA: Calibrated Alignment.}
\label{tab:evolutionary-necessity}
\setlength{\tabcolsep}{0.5em}
\scriptsize
\begin{tabular}{@{}lll@{}}
\toprule
\textbf{Method} & \textbf{Components} & \textbf{Bottleneck} \\
\midrule
\begin{tabular}[t]{@{}l@{}}
\textbf{Baseline} \\
(e.g., SRe$^2$L)
\end{tabular} & Full Soft Labels & 
\begin{tabular}[t]{@{}l@{}}
\textbf{Raw Data Size} \\
Label Storage $>$ Image Storage by 200$\times$
\end{tabular} \\
\midrule
\begin{tabular}[t]{@{}l@{}}
Single Technique
\end{tabular} & 
\begin{tabular}[t]{@{}l@{}}
\textbf{P} \textit{or} \textbf{Q}
\end{tabular} & 
\begin{tabular}[t]{@{}l@{}}
\textbf{Auxiliary Data Dominance} \\
Indices/Augmented Info. dominate total storage
\end{tabular} \\
\midrule
\begin{tabular}[t]{@{}l@{}}
Naive Combine
\end{tabular} & 
\begin{tabular}[t]{@{}l@{}}
P \textbf{+} Q 
\end{tabular} & 
\begin{tabular}[t]{@{}l@{}}
\textbf{Accuracy Collapse} \\
Sparse supervision \& mismatched distributions
\end{tabular} \\
\midrule
\begin{tabular}[t]{@{}l@{}}
\textbf{LPQLD} (Ours)
\end{tabular} & 
\begin{tabular}[t]{@{}l@{}}
P (+\textbf{DKR}) + \\ 
Q (+\textbf{CA})
\end{tabular} & 
\begin{tabular}[t]{@{}l@{}}
\textbf{Both Resolved} \\
DKR fixes sparsity, CA fixes mismatch
\end{tabular} \\
\bottomrule
\end{tabular}
\end{table}

Table~\ref{tab:lpld-lpqld-comparison} summarizes the key differences between LPLD and LPQLD across three categories, highlighting LPQLD's innovations in label efficiency (quantization) and supervision diversity (DKR and CA).
Table~\ref{tab:evolutionary-necessity} shows why these components are essential solutions to fundamental bottlenecks.
The baseline suffers from impractical label storage (200$\times$ larger than images) and single techniques (\textbf{P} or \textbf{Q}) hit auxiliary data dominance.
Naively combining them (P \textbf{+} Q) is possible, but it leads to accuracy collapse due to sparse supervision and mismatched distributions.
Our LPQLD resolves both bottlenecks: P (+\textbf{DKR}) addresses sparsity by extracting varied supervisory signals via temperature annealing, while Q (+\textbf{CA}) fixes distribution mismatch by dynamically adjusting student temperature.
This establishes LPQLD as a systematic framework where each component addresses specific bottlenecks induced by the interaction of pruning and quantization.

The contributions of this work can be summarized as follows:
\begin{enumerate}[leftmargin=*, itemsep=0pt, topsep=0pt]
    \item We identify \textbf{supervision diversity} as the fundamental bottleneck at high compression rates and provide a unified LPQLD framework for soft label compression.
    
    \item We decompose \textbf{supervision diversity} into \textbf{augmentation-per-image diversity} and \textbf{label-per-augmentation diversity}, and propose two modules combined with Pareto front analysis to enhance both types of diversity.
    
    \item LPQLD achieves unprecedented storage efficiency (\textbf{78$\times$} on ImageNet-1K, \textbf{500$\times$} on ImageNet-21K) while maintaining or exceeding LPLD's performance across diverse architectures and distillation methods.
\end{enumerate}

\section{Related Works}
\label{sec: related-work}

\textbf{Large-Scale Dataset Distillation.}
Recent advancements in dataset distillation have pushed towards handling \textbf{large-scale datasets} like ImageNet-1K~\cite{deng2009imagenet} and beyond. While early methods struggled with the memory demands and scale~\cite{cazenavette2022distillation, kimICML22}, techniques like TESLA~\cite{cui2023scaling} and particularly multi-stage methods such as SRe$^2$L~\cite{yin2023squeeze} have enabled distillation of full ImageNet-1K. SRe$^2$L notably decouples the process into squeezing, recovering (image synthesis), and relabeling phases. Consequently, significant research effort has concentrated on improving the \textbf{quality of the synthesized images} during the recovery phase, exploring various strategies to generate more informative condensed datasets including through optimization~\cite{yin2023dataset, sun2023diversity, dwa2024neurips, shao2023generalized, zhong2025going, bo2025understanding, cui2025dataset, asano2023the, shen2024deltsimplediversitydrivenearlylate, liu2023dataset, he2024you, loo2024large, wang2024emphasizing, zhou2024self} or through generative methods~\cite{su2024d, li2024generative, chen2025influence}.
To attain decent performance, the \textbf{relabeling phase}, responsible for generating extensive soft labels from a teacher model under numerous augmentations, plays a crucial role.
However, soft labels have received comparatively less attention regarding their efficiency~\cite{zhang2024breaking, xiao2025rethinkdc}.
The relabeling phase, while crucial for achieving state-of-the-art performance, introduces a significant \textbf{storage bottleneck} due to the massive volume of soft label data required, diverging from the goal of dataset distillation.

\begin{figure*}[t]
    \centering
    \includegraphics[width=1\linewidth]{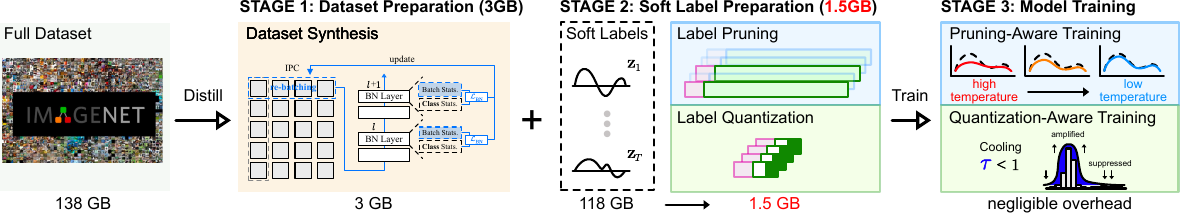}
    \caption{Pipeline overview of LPQLD. In STAGE 1, we generate a synthetic dataset by matching class-wise BN statistics.
    In STAGE 2, we generate corresponding soft labels. To reduce soft label storage, label pruning and label quantization are conducted.
    In STAGE 3, compression-aware model training is performed with distilled datasets and compressed soft labels.}
    \label{fig:pipeline}
\end{figure*}

\textbf{Soft Label Compression.}
Despite the extensive storage requirements of soft labels in methods like SRe$^2$L~\cite{yin2023squeeze}, directly addressing this \textbf{label storage bloat} has been largely overlooked in dataset distillation research. Some works have investigated how to better leverage existing soft labels during student training~\cite{qin2024label, shen2023ferkd}, but without focusing on reducing their storage footprint, and in some cases even increasing it~\cite{kang2024label}.

Teddy~\cite{yu2025teddy} proposes using a proxy model for dynamic label generation, which eliminates pre-computation storage but introduces efficiency trade-offs and additional training overhead.
FKD~\cite{shen2022fast}, though designed for traditional knowledge distillation, employs label quantization by storing both indices and quantized post-softmax logits for recovery during training. However, it has limitations: (1) it was designed for full datasets rather than distilled ones, creating unique challenges in large-scale dataset distillation where limited image supervision is available, and (2) it does not address the storage requirements of auxiliary augmentation information.
In this paper, we follow the naming of ``label quantization'' to be consistent with previous works~\cite{shen2022fast,xiao2024lpld}.

\section{Method}
\label{sec:method-implicit}

In this section, we present our Label Pruning and Quantization for Large-scale Distillation (LPQLD) approach, which addresses the soft label storage bottleneck while maintaining or improving performance. As illustrated in Fig.~\ref{fig:pipeline}, our method consists of three key components: (1) diverse synthetic dataset generation to enhance \textit{image diversity}, (2) soft label compression via pruning and quantization to enhance \textit{supervision diversity}, and (3) compression-aware training with dynamic knowledge reuse and calibrated student-teacher alignment. This integrated pipeline effectively reduces storage requirements while preserving the knowledge contained in soft labels.

\begin{figure}[t]
    \centering
    \includegraphics[width=.9\linewidth]{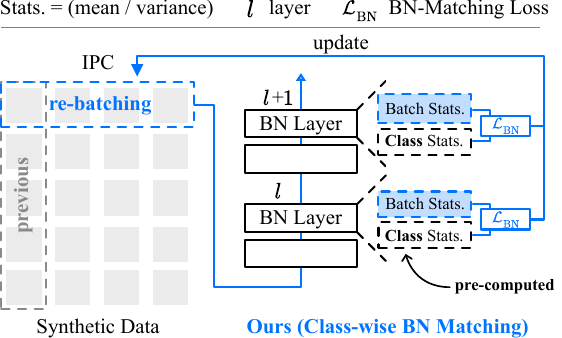}
    \caption{
        Illustration of our diverse synthetic dataset generation. Our method encourages collaboration among images of the same class, yielding a natural separation between classes. Here, the classification loss is omitted for illustration.
    }
    \label{fig:bn-loss}
    \vspace{-0.3cm}
\end{figure}

\subsection{Preliminary}
\textbf{Synthetic Dataset Generation.}
We begin by considering the standard Batch Normalization (BN) transformation:
\begin{equation}
    y = \gamma \left( \frac{\boldsymbol{x} - \mu}{\sqrt{\sigma^2 + \epsilon}} \right) + \beta,
\end{equation}
where $\boldsymbol{x}$ represents the input feature maps to the BN layer, $y$ is the normalized output, $\gamma$ and $\beta$ are learnable scaling and shifting parameters, while $\mu$ and $\sigma^2$ represent the mean and variance of the inputs. The constant $\epsilon$ is introduced for numerical stability to avoid division by zero. During training, running statistics (mean and variance) are maintained and later used during inference to approximate the unavailable true statistics of testing data.

Inspired by SRe$^2$L~\cite{yin2023dataset} and DeepInversion~\cite{yin2020dreaming}, we generate synthetic imagery by aligning the BN statistics of intermediary layers. In particular, the target is to minimize the discrepancy between the BN statistics computed from the synthetic images $\widetilde{\boldsymbol{x}}$ and the stored running estimates from the full dataset $\mathcal{T}$:

{\footnotesize%
\begin{equation}
\begin{aligned}
\mathcal{L}_{\mathrm{BN}}(\widetilde{\boldsymbol{x}}) & =\sum_l\left\|\mu_l(\widetilde{\boldsymbol{x}})-\mathbb{E}\left(\mu_l \mid \mathcal{T}\right)\right\|_2+\sum_l\left\|\sigma_l^2(\widetilde{\boldsymbol{x}})-\mathbb{E}\left(\sigma_l^2 \mid \mathcal{T}\right)\right\|_2 \\
& \approx \sum_l\left\|\mu_l(\widetilde{\boldsymbol{x}})-\mathbf{BN}_l^{\mathrm{RM}}\right\|_2+\sum_l\left\|\sigma_l^2(\widetilde{\boldsymbol{x}})-\mathbf{BN}_l^{\mathrm{RV}}\right\|_2,
\end{aligned}
\end{equation}
}%
where $\mathbf{BN}_l^{\mathrm{RM}}$ and $\mathbf{BN}_l^{\mathrm{RV}}$ denote the running mean and variance that serve as surrogates for the expected values. Here, $\mu_l(\widetilde{\boldsymbol{x}})$ and $\sigma_l^2(\widetilde{\boldsymbol{x}})$ are the per-layer statistics computed on a batch of synthetic images $\widetilde{\boldsymbol{x}}$.

This BN loss acts as a regularizer to the classification loss $\mathcal{L}_{\mathrm{CE}}$, yielding the joint objective:
\begin{equation}
\underset{\widetilde{\boldsymbol{x}}}{\arg \min} \quad \underbrace{\ell\left(\boldsymbol{\theta}_{\mathcal{T}}\left(\widetilde{\boldsymbol{x}}\right), \boldsymbol{y}\right)}_{\mathcal{L}_{\mathrm{CE}}}+\alpha \cdot \mathcal{L}_{\mathrm{BN}}\left(\widetilde{\boldsymbol{x}}\right),
\label{eq:sre2l-objective}
\end{equation}
where $\boldsymbol{\theta}_{\mathcal{T}}$ is the model pretrained on the original dataset and the scalar $\alpha$ adjusts the influence of the BN regularization.

\textbf{Image Relabeling.}
In model distillation, huge teachers impose additional training burdens to computing demands on top of the existing two bottlenecks, deep network propagation and data loading~\cite{shen2022fast}.
FKD~\cite{shen2022fast} introduces the relabel phase that pre-computes and stores hyper-parameters of augmentations $\mathcal{\{F\}}$ together with soft label sets $\mathcal{\{P\}}$, where $\mathcal{\{P\}} = \text{softmax}(\mathbf{z}, \tau)$ represents probability distributions derived from the logit outputs $\mathbf{z}$ of the teacher model through softmax with temperature $\tau$. The storage size of augmentation hyper-parameters $|\mathcal{\{F\}}|$ is negligible compared to the size of full soft labels $|\mathcal{\{P\}}|$.
During the model training phase, the stored label files are loaded; images undergo the same augmentation by applying $\mathcal{\{F\}}$, making the supervision from $\mathcal{\{P\}}$ valid.
To efficiently store the soft labels, label quantization is applied, and merely top-$k$ (e.g., $k=5$) soft labels are selected for each image per augmentation. 
Note that both indices and values are stored, so the actual storage for top-$k$ is equivalent to $2\times k$.
For the ImageNet-1K dataset, this cuts down the soft label storage $\mathcal{\{P\}}$ to 100$\times$ by taking top-$5$ predictions. 

SRe$^2$L~\cite{yin2023squeeze} adopts the idea of relabeling process~\cite{shen2022fast} into dataset distillation, and it is the first to achieve effective large-scale (ImageNet-1K-scale) dataset distillation.
However, due to the large discrepancy in the dataset size between full ImageNet-1K and distilled ImageNet-1K, the label quantization process does not effectively apply.
The full ImageNet-1K has over 1000K images to accommodate information loss from each soft label, while distilled ImageNet-1K has only 10K images (i.e., 10 images per class). 
As a consequence, by giving up the label quantization, soft label size $|\mathcal{\{P\}}|$ is about 40$\times$ of the distilled dataset size $|\mathcal{S}|$.

\subsection{Image Diversity via Diverse Synthetic Dataset Generation}
\label{sec: similarity-analysis}

The excessive storage requirements for soft labels in current dataset distillation methods fundamentally contradict the core motivation of dataset distillation—creating compact, efficient dataset representations. We hypothesize that a key underlying factor is \textit{insufficient image diversity} within synthetic image classes, which necessitates extensive augmentation and corresponding soft labels to provide varied supervision. Before introducing our compression techniques, we first establish metrics to quantify this diversity problem and then propose methods to address it at the image synthesis stage.

\subsubsection{Intra-Class Diversity: Feature Cosine Similarity}
A key measure of diversity is the similarity of images within the same class. We quantify this by calculating the cosine similarity of their feature representations. Lower cosine similarity indicates that the images are less alike, thus the intra-class diversity is higher. This observation is formally stated as:
\begin{proposition}
A reduction in cosine similarity among features of images within the same class reflects increased diversity.
\end{proposition}
Mathematically, the cosine similarity between the features of two images is expressed as:
\begin{equation}
\mathrm{cos~similarity} :=\frac{\sum_{i=1}^n f(\boldsymbol{\widetilde{x}_{c,i}})~f(\boldsymbol{\widetilde{x}^\prime}_{c,i})}{\sqrt{\sum_{i=1}^n f(\boldsymbol{\widetilde{x}_{c,i}})^2} \sqrt{\sum_{i=1}^n f(\boldsymbol{\widetilde{x}^\prime}_{c,i})^2}},
\end{equation}
where $\boldsymbol{\widetilde{x}_c}$ and $\boldsymbol{\widetilde{x}^\prime}_c$ are two images from class $c$, $f(\cdot)$ is the feature extractor, and $n$ is the dimensionality of the feature vector.

\subsubsection{Inter-Dataset Consistency: Maximum Mean Discrepancy (MMD)}
While intra-class differences are important, it is equally crucial that the synthetic dataset faithfully captures the original data distribution. We assess this through Maximum Mean Discrepancy (MMD), which compares the feature distributions of the synthetic and real datasets:
\begin{proposition}
A lower MMD value indicates that the synthetic dataset closely approximates the original dataset's range of features, thus reflecting higher diversity.
\end{proposition}
The empirical estimation of MMD is defined as~\cite{zhang2024m3d, qin2024distributional}:
\begin{equation}
\operatorname{MMD}^2\left(\boldsymbol{P}_{\mathcal{T}}, \boldsymbol{P}_{\mathcal{S}}\right)=\hat{\mathcal{K}}_{\mathcal{T}, \mathcal{T}}+\hat{\mathcal{K}}_{\mathcal{S}, \mathcal{S}}-2 \hat{\mathcal{K}}_{\mathcal{T}, \mathcal{S}},
\end{equation}

\begin{equation}
\hat{\mathcal{K}}_{X, Y}=\frac{1}{|X| \cdot |Y|} \sum_{i=1}^{|X|} \sum_{j=1}^{|Y|} \mathcal{K}\left(f\left(x_i\right), f\left(y_j\right)\right),
\end{equation}
where $\{x_i\}_{i=1}^{|X|} \sim X$, $\{y_j\}_{j=1}^{|Y|} \sim Y$, $\mathcal{T}$ and $\mathcal{S}$ denote the real and synthetic datasets, $\mathcal{K}$ is a reproducing kernel (e.g., Gaussian), and $f(\cdot)$ maps inputs to their feature representations.

\begin{table}[t]
    \footnotesize
    \centering
    \setlength{\tabcolsep}{0.6em}
    \caption{The cosine similarity between image features. It is the average of 1K class on the synthetic ImageNet-1K dataset. Features are extracted using pretrained ResNet-18.}
    \begin{tabular}{@{}llll|l@{}}
    \toprule
    IPC    & SRe$^2$L         & CDA             & Ours              & Full Dataset \\ \midrule
    50  & $0.841 \pm 0.023$ & $0.816 \pm 0.026$ & $0.796 \pm 0.029$ & \multirow{3}{*}{$0.695 \pm 0.045$}\\
    100 & $0.840 \pm 0.016$ & $0.814 \pm 0.019$ & $0.794 \pm 0.021$ & \\
    200 & $0.839 \pm 0.011$ & $0.813 \pm 0.013$ & $0.793 \pm 0.015$ & \\ \bottomrule \\
    \end{tabular}
    \label{tab: cosine-similarity}
    \vspace{-0.5cm}
\end{table}

\begin{figure}[t]
    \centering
    \includegraphics[width=.55\linewidth]{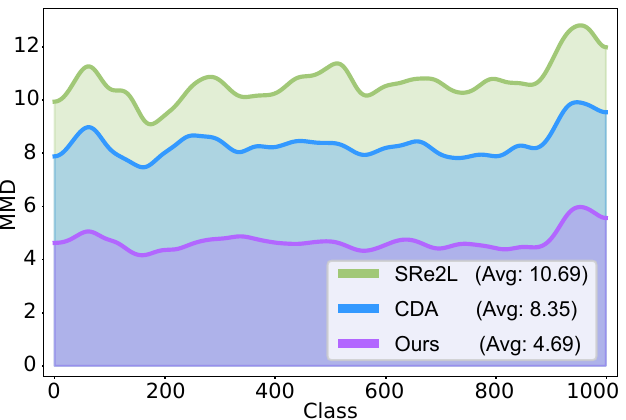}
    \caption{Comparison of Maximum Mean Discrepancy (MMD).}
    \label{fig:mmd}
    \vspace{-0.4cm}
\end{figure}

\subsubsection{Diversifying the Datasets by Class-wise Supervision}
\label{sec:diverse-sample}

In previous objective (Eq.~\ref{eq:sre2l-objective}), a subset of classes $\mathcal{B}_c$ is sampled to match global BN statistics, leading to independent generation for each image of a class and, consequently, lower intra-class diversity. Drawing inspiration from He \etal~\cite{he2024multisize}, we argue that images in the same class can benefit from working cooperatively when generating multiple images per class.

\textbf{Step 1: Intra-Class Re-batching.}
To promote image collaboration, we sample multiple images from the same class and apply class-specific supervision~\cite{zhao2021datasetGM, kimICML22}. This change, illustrated in Fig.~\ref{fig:bn-loss}, ensures that images are updated together rather than in isolation.

\textbf{Step 2: Incorporating Class-wise BN Supervision.}
Since the global running statistics may not be optimal for class-specific information, we propose maintaining separate BN statistics for each class. The additional storage cost is minimal even for large-scale datasets (e.g., ImageNet-1K; see Appendix~\ref{appendix:storage-class-stats}).

\textbf{Step 3: Designing the Class-wise Objective.}
We adjust the original objective by introducing two losses. First, we compute the classification (cross-entropy) loss using global BN statistics to ensure overall effective supervision. Second, the BN loss is now defined by comparing the synthetic images’ BN statistics with the class-specific running statistics. The updated class-wise objective is:

{\footnotesize
\vspace{-.5cm}
\begin{equation}
\label{eq:ours}
\begin{aligned}
\underset{\widetilde{\boldsymbol{x}}_c}{\arg \min} \Biggl( &\overset{\text{Cross-Entropy Loss with Global BN}}{\overbrace{- \sum_{i=1}^N y_{c,i} \log \left( \text{softmax}\Biggl( \boldsymbol{\theta}_{\mathcal{T}}\Biggl(\frac{\widetilde{\boldsymbol{x}}_{c,i} - \mathbf{BN}^\mathrm{RM}_{\text{global}}}{\sqrt{\mathbf{BN}^\mathrm{RV}_{\text{global}} + \epsilon}}\Biggr)\Biggr)_c \right)}} \\
&+ \alpha \cdot \underset{\text{BN Loss with Class-wise BN}}{\underbrace{\sum_l \left( \left\|\mu_l(\widetilde{\boldsymbol{x}}_c)-{\mathbf{BN}_{l,c}^{\mathrm{RM}}}\right\|_2 + \left\|\sigma_l^2(\widetilde{\boldsymbol{x}}_c)-{\mathbf{BN}_{l,c}^{\mathrm{RV}}}\right\|_2 \right)}} \Biggr).
\end{aligned}
\end{equation}
}%

Note that while the BN loss utilizes class-specific statistics, the computation of logits for the cross-entropy loss still relies on global statistics. This design choice is crucial because adjusting only $\mu$ and $\sigma$ without refining $\gamma$ and $\beta$ could degrade model performance.

\textbf{Superior Performance in Similarity Metrics.} 
Thanks to the adjustments described in Section~\ref{sec:diverse-sample}, our synthetic dataset demonstrates both lower intra-class feature cosine similarity (as evidenced in Table~\ref{tab: cosine-similarity}) and a much reduced MMD (see Fig.~\ref{fig:mmd}) compared to prior methods. These improvements indicate enhanced \textit{image diversity}, where the synthetic images capture a broader and more representative feature distribution of the original dataset. With this diverse synthetic dataset established, we next address the issue of superfluous soft labels.

\subsection{Soft Label Pruning and Quantization}
\label{sec: label-pruning-quantization}

\begin{figure*}[t]
    \centering
    \includegraphics[width=.9\linewidth]{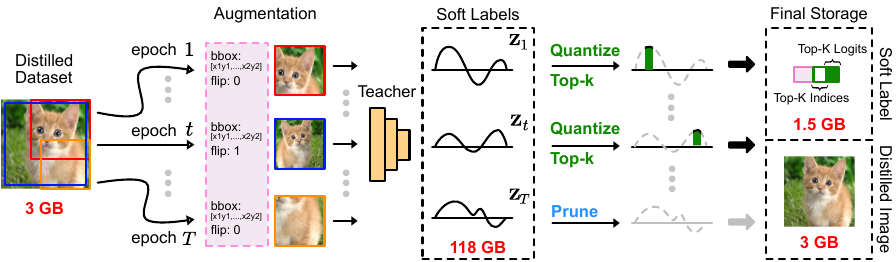}
    \caption{
    Illustration of the soft label generation process incorporating label pruning and label quantization techniques. The process begins by augmenting each dataset image for $T$ epochs, with each augmented image producing different outputs from a fixed pretrained teacher. Traditional approaches store all augmentations information and soft labels across all $T$ epochs. \textbf{Label pruning} retains only $(1-p) \times T$ epochs out of the total $T$ epochs, where $p < 1$ represents the pruning rate. \textbf{Label quantization}, an orthogonal process, reduces the storage by selecting only the Top-$k$ logits $\mathbf{z}$ to store. Note that we store the \textit{before softmax} output logits.
    }
    \label{fig:label-prune-quantize}
\end{figure*}

\begin{figure*}[t]
    \centering
    \includegraphics[width=.9\linewidth]{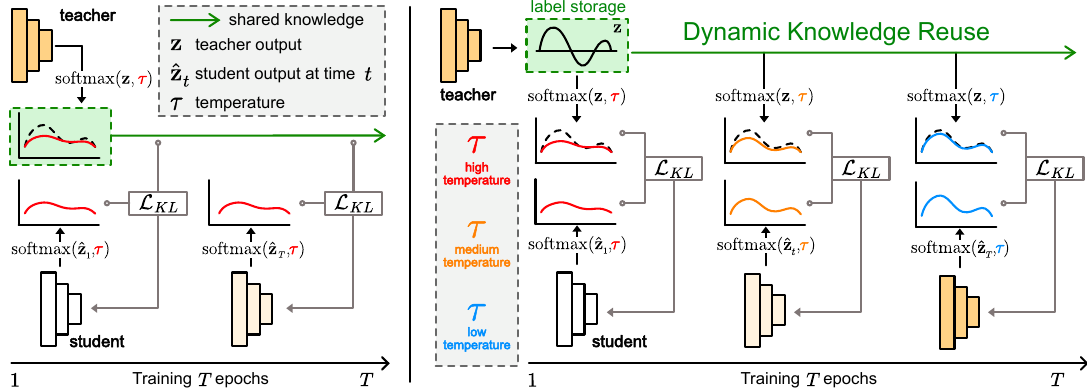}
    \caption{Illustration of knowledge reuse due to pruning that leads to insufficient amount of teacher information for all $T$ epochs of model distillation training.
    (\textbf{Left}) Naive Knowledge Reuse: despite the training stages, the supervision from teachers is fixed for a specific pre-stored augmentation.
    (\textbf{Right}) Dynamic Knowledge Reuse: to allow the same augmentation to have diverse supervision signals, we set a different temperature during \textit{softmax} operation for the same augmentation at different training epochs. In particular, we use an annealing temperature scheduler. 
    Note: \textit{before softmax} output logits are stored, and such dynamic temperature assignment does not introduce additional storage costs.
    }
    \label{fig:label-prune}
\end{figure*}

\textbf{Label Pruning.}
As illustrated in Fig~\ref{fig:label-prune-quantize} ({\color{pruningblue} blue} part), label pruning prunes both augmentation hyper-parameters $\mathcal{\{F\}}$ and soft labels $\mathcal{\{P\}}$.
For a model distillation training of $T=300$ epochs, the theoretical maximum pruning rate would be $300\times$ to ensure each image undergoes at least a single augmentation.
Benefiting from the enhanced \textit{image diversity} that implicitly reduces the number of different augmentations needed, we adopt random label pruning.
Compared to metric-based label pruning that requires selection from $T$ epochs, random pruning enjoys acceleration in the label generation phase. 
In particular, we generate labels for $(1-p) \times T$ epochs, meaning if there are 300 epochs $T$ of training in total and pruning rate $p$ is 0.9, 3 epochs of model propagation are required and actual $10\times$ speed-up is achieved.
The soft label generation process can be time-consuming especially for limited CPU resources and a low-speed disk since extensive memory movement and system IO operations are executed.

\textbf{Label Quantization.}
Orthogonal to label pruning, label quantization further reduces storage requirements to enhance \textit{augmentation-per-image diversity}.
FKD~\cite{shen2022fast} introduces different label quantization approaches, including Hardening, Smoothing, Marginal Smoothing (MS), and Marginal Re-normalization (MR).
Our method stores the set of top-$k$ pre-softmax logits $\mathbf{\{z\}}$ to enable dynamic temperature scheduling for enhancing \textit{label-per-augmentation diversity}.
By incorporating the top-k selection, the post-softmax distribution $\mathcal{P}_{i}(\mathbf{z}_i^k, \mathcal{F}_i, \tau)$ is defined as:
\begin{equation}
\begin{aligned}
    \mathcal{P}_{i}(\mathbf{z}_i^k, \mathcal{F}_i, \tau) &= \mathrm{softmax}(\mathbf{z}_i^k, \tau), \quad\mathrm{where} \\
    \mathbf{z}_i^k &= \mathrm{topk}(\theta_{\mathcal{T}}\left(\mathcal{A}(x_{i}, \mathcal{F}_{i})\right), k).
\end{aligned}
\label{eq: z-quantization}
\end{equation}

\subsection{Supervision Diversity 1 via Dynamic Knowledge Reuse}
\label{sec: dynamic-knowledge-reuse}
After label pruning, the remaining labels must be reused across multiple training epochs to provide sufficient supervision throughout the entire training process. Simply repeating the same supervision signals is possible but limits \textit{label-per-augmentation diversity}. The challenge is maximizing knowledge transfer from this reduced set of labels.

\textbf{Dynamic Knowledge Reuse.}
We propose Dynamic Knowledge Reuse, which leverages temperature annealing to generate diverse supervision signals from the same pruned labels, enhancing \textit{label-per-augmentation diversity}. Temperature in softmax controls the entropy of probability distributions, higher temperatures produce smoother distributions that reveal more class similarities, while lower temperatures create sharper distributions focusing on likely classes.

The key insight is that by varying temperature over time, we can create a curriculum learning effect that enhances \textit{label-per-augmentation diversity}: starting with higher temperatures to learn broad class relationships, then gradually transitioning to lower temperatures to refine specific class boundaries.

\textbf{Storage-Efficient Implementation.}
Naively implementing temperature annealing would require storing probability distributions for each temperature, which is prohibitively expensive. Instead, we store only the pre-softmax logits $\mathbf{z}_i = \theta_{\mathcal{T}}(\mathcal{A}(x_{i}, \mathcal{F}_{i}))$ once and dynamically compute the probability distribution for any temperature during training as shown in Fig.~\ref{fig:label-prune}:
\begin{equation}
    p_i^{(\tau_t)} = \frac{\exp(z_i/\tau_t)}{\sum_{j=1}^C \exp(z_j/\tau_t)}.
\end{equation}

This approach provides significant storage savings while enabling exact recovery of any temperature-specific distribution for \textit{label-per-augmentation diversity}. For a $C$-class problem, we store only $C$ values per image instead of $C$ values for each of the many temperature settings we would use throughout training.

\textbf{Theoretical Justification for Label-Pruned Temperature Annealing.}
Annealing KD~\cite{jafari2021annealing} established that temperature annealing creates a beneficial learning curriculum by leveraging the VC-dimension theory~\cite{vapnik1999overview, mirzadeh2020improved}. The key theoretical advantage they proved was that temperature annealing leads to better convergence compared to standard KD. This can be expressed through the inequality~\cite{jafari2021annealing, lopez2015unifying}:
\begin{equation}
R_{\tau}(f_s^{\tau}) - R_{\tau}(f_t^{\tau}) \leq O\left(\frac{|\mathcal{H}_s|_c}{N^{\alpha_{st}^{\tau}}}\right) + \varepsilon_{st}^{\tau} \leq O\left(\frac{|\mathcal{H}_s|_c}{N^{\alpha_{st}}}\right) + \varepsilon_{st}
\end{equation}
where $R(\cdot)$ is the expected error, $f_s^{\tau}$ and $f_t^{\tau}$ represent student and teacher functions at temperature $\tau$, and $\alpha_{st}^{\tau} \geq \alpha_{st}$ indicates better convergence rates with annealing.

In our label-pruned setting where we reuse a fraction $\lambda$ of teacher predictions, we can adapt this inequality to account for repeated supervision:
\begin{equation}
R_{\tau}(f_s^{\tau}) - R_{\tau}(f_t^{\tau}) \leq O\left(\frac{|\mathcal{H}_s|_c}{(k(\tau) \cdot |\mathcal{D}_t^{\lambda}|)^{\alpha_{st}^{\tau}}}\right) + \varepsilon_{st}^{\tau}
\end{equation}
where $\mathcal{D}_t^{\lambda}$ is our pruned subset with pruning ratio $\lambda$ and $k(\tau)$ represents the effective number of times samples are reused up to temperature $\tau$.
This adapted bound thus suggests that the benefits of annealing are preserved, as the effective sample reuse $k(\tau)$ helps to maintain a strong learning signal. 

The benefits of Annealing KD remain valid in our label-pruned setting for two main reasons:
\textbf{First}, the core mechanism of temperature annealing does not depend on having unique teacher predictions for each training instance. The same logits viewed through different temperatures still create the desired curriculum effect from smooth to sharp distributions, effectively creating \textit{label-per-augmentation diversity} in the prediction space.
\textbf{Second}, high \textit{label-per-augmentation diversity} from repeatedly using the same predictions through different temperatures compensates for the reduced quantity of supervision. While we store only a fraction $\lambda$ of the original teacher predictions, reusing them with an annealing temperature schedule allows the student to progressively learn more difficult aspects of the same predictions.

\subsection{Supervision Diversity 2 via Calibrated Alignment}
\label{sec: student-teacher-alignment}

\begin{figure}[t]
    \centering
    \includegraphics[width=1\linewidth]{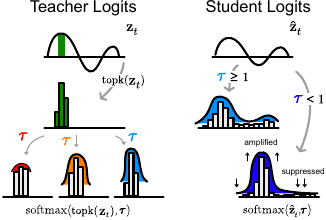}
    \caption{Illustration of distribution-aware temperature scaling for student networks.
    (\textbf{Left}) Soft label distribution of teacher after \textit{softmax}. \textit{Softmax} applies to Top-$k$ selected logits, and remaining logits assign a probability of 0. 
    (\textbf{Right}) Assigning students a `colder' temperature ($\tau < 1$) better aligns with the soft label distribution of the teacher. 
    }
    \label{fig:label-quantize}
    \vspace{-1em}
\end{figure}

Label quantization enables enhanced \textit{augmentation-per-image diversity} by reducing per-label storage, allowing more augmentation-label pairs under the same storage budget. Unlike label pruning, quantization fundamentally changes the teacher's prediction distribution $\mathcal{P}$. By applying softmax on quantized logits $\mathbf{z}_i^k$, probabilities are re-normalized within the top-$k$ selected classes, creating a sharper distribution with lower entropy than the unquantized distribution.

\textbf{Theoretical justification for $\hat{\tau} < 1$:} For the student to match the quantized teacher distribution when minimizing KL divergence, the following ratio equality must hold for any two classes $i, j$ in the top-$k$ set (see Appendix~\ref{appendix:kl-ratio-proof} for detailed proof):
\begin{equation}
\frac{\mathcal{\hat{P}}_i(\mathbf{\hat{z}}_i, \mathcal{F}_i, \hat{\tau})}{\mathcal{\hat{P}}_j(\mathbf{\hat{z}}_j, \mathcal{F}_j, \hat{\tau})} = \frac{\mathcal{P}_i(\mathbf{z}_i^k, \mathcal{F}_i, \tau_t)}{\mathcal{P}_j(\mathbf{z}_j^k, \mathcal{F}_j, \tau_t)}
\end{equation}

Expressing this in terms of softmax and taking logarithms:
\begin{equation}
\frac{\mathbf{\hat{z}}_i - \mathbf{\hat{z}}_j}{\hat{\tau}} = \log\left(\frac{\mathcal{P}_i(\mathbf{z}_i^k, \mathcal{F}_i, \tau_t)}{\mathcal{P}_j(\mathbf{z}_j^k, \mathcal{F}_j, \tau_t)}\right)
\end{equation}

Quantization produces a sharper distribution by concentrating probability mass on fewer classes. To match this sharper distribution under practical logit constraints, the temperature $\hat{\tau}$ must be below 1 to amplify existing logit differences. We determine the optimal temperature by minimizing KL divergence:
\begin{equation}
    \argmin_{\hat{\tau}^*} \sum_{i=1}^{n \in \mathcal{B}} \mathcal{P}_i(\mathbf{z}_i^k, \mathcal{F}_i, \tau_t) \log
    \Big(
    \frac{\mathcal{P}_i(\mathbf{z}_i^k, \mathcal{F}_i, \tau_t)}{\mathcal{\hat{P}}_i(\mathbf{\hat{z}}_i, \mathcal{F}_i, \hat{\tau})}
    \Big).
    \label{eq: optimization-problem}
\end{equation}

\begin{figure}[t]
    \centering
    \includegraphics[width=.9\linewidth]{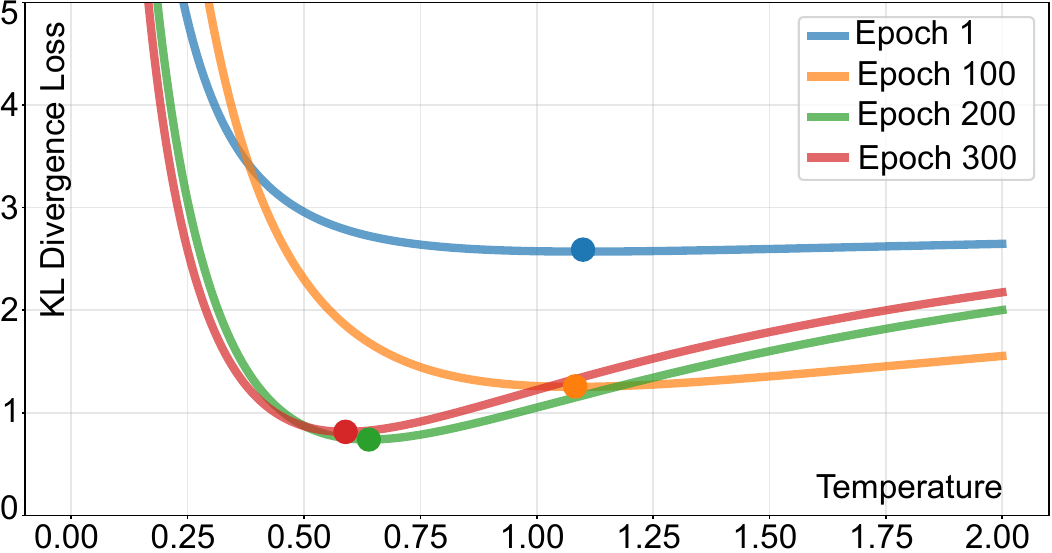}
    \caption{Loss landscape of KL divergence showing a clear minimum within $\tau \in [0,1]$.}
    \label{fig:loss-landscape}
\end{figure}

\textbf{Optimization via Grid Search.}
We solve this optimization using grid search over $\hat{\tau} \in [0,1]$ because: (1) computing gradients through nested softmax operations is expensive, and (2) the KL divergence landscape is well-behaved with respect to $\tau$ (Fig.~\ref{fig:loss-landscape}). We calculate the optimal temperature using the last batch of the $t$-th epoch and apply it to the $t+1$ epoch, as outlined in Algorithm~\ref{algo: student-temperature}.

\textbf{Computational Cost.}
The grid search incurs minimal overhead as it only recomputes softmax values and evaluates KL divergence under different temperatures. Importantly, it operates on a single batch (the last batch of each epoch) rather than the entire dataset, making the cost independent of dataset size. The total additional computational cost across all 300 training epochs is negligible (a few seconds in total).

\begin{algorithm}[t]
\caption{Dynamic Student Temperature via Grid Search}
\label{algo: student-temperature}
\begin{algorithmic}[1]
\STATE \textbf{Input:} Teacher logits $\mathcal{P}$, Student logits $\mathbf{\hat{z}}$
\STATE Initialize $\mathbb{T} \gets \{0.01, 0.02, \dots, 1.00\}$, $\mathcal{L}_{\mathrm{min}} \gets \infty$
\FOR{$\hat{\tau} \in \mathbb{T}$}
    \STATE $\mathcal{\hat{P}} \gets \mathrm{softmax}(\mathbf{\hat{z}} / \hat{\tau})$
    \STATE $\mathcal{L}_{\hat{\tau}} \gets \mathrm{KL}(\mathcal{P} \| \mathcal{\hat{P}})$
    \IF{$\mathcal{L}_{\hat{\tau}} < \mathcal{L}_{\mathrm{min}}$}
        \STATE $\mathcal{L}_{\mathrm{min}} \gets \mathcal{L}_{\hat{\tau}}$, $\hat{\tau}^* \gets \hat{\tau}$
    \ENDIF
\ENDFOR
\STATE \textbf{Output:} Optimal student temperature $\hat{\tau}^*$ for next epoch
\end{algorithmic}
\end{algorithm}
\section{Experiments}

\newcolumntype{C}[1]{>{\centering\arraybackslash}p{#1}}
\newlength{\lpqldwidth}
\setlength{\lpqldwidth}{1.1cm}  
\newlength{\normalwidth}
\setlength{\normalwidth}{0.54\lpqldwidth}  

\begin{table*}[t]
\centering
\footnotesize
\setlength{\tabcolsep}{0.20em}
\caption{ImageNet-1K label pruning result. Our method consistently shows a better performance under various pruning ratios. \colorbox{lightblue}{LPQLD} uses label pruning, and \colorbox{lightgreen}{LPQLD} uses both label pruning and label quantization. The base pruning rate for LPQLD is 10$\times$. $\{\mathbf{z}, \mathcal{F}\}$ denotes total storage of both augmentation information and soft labels. The validation model is ResNet-18. $^\dag$ denotes the reported results.}
\label{tab:label_pruning_main}
\begin{tabular}{l|C{\normalwidth}C{\normalwidth}C{\normalwidth}|C{\normalwidth}C{\normalwidth}C{\normalwidth}C{\lpqldwidth}|C{\normalwidth}C{\normalwidth}C{\normalwidth}C{\lpqldwidth}C{\lpqldwidth}|C{\normalwidth}C{\normalwidth}C{\normalwidth}C{\lpqldwidth}C{\lpqldwidth}|C{\lpqldwidth}|C{\lpqldwidth}}
    \toprule
    & \multicolumn{3}{c|}{$\sim$ 1$\times$}         & \multicolumn{4}{c|}{$\sim$ 10$\times$}        & \multicolumn{5}{c|}{$\sim$ 20$\times$}        & \multicolumn{5}{c|}{$\sim$ 40$\times$}       & \multicolumn{1}{c|}{$\sim$ 80$\times$} & \multicolumn{1}{c}{$\sim$ 200$\times$} \\
    {\scriptsize ResNet-18} & {\scriptsize SRe$^2$L} & {\scriptsize CDA} & {\scriptsize LPLD} & {\scriptsize SRe$^2$L} & {\scriptsize CDA} & {\scriptsize LPLD} & {\scriptsize LPQLD} & {\scriptsize SRe$^2$L} & {\scriptsize CDA} & {\scriptsize LPLD} & \multicolumn{2}{c|}{{\scriptsize LPQLD}} & {\scriptsize SRe$^2$L} & {\scriptsize CDA} & {\scriptsize LPLD} & \multicolumn{2}{c|}{{\scriptsize LPQLD}} & {\scriptsize LPQLD} & {\scriptsize LPQLD} \\ \midrule
    \multicolumn{19}{l}{\textbf{IPC10} (168 M)} \\ \midrule
    Acc (\%)     & 20.1     & 33.3 & 34.6 & 18.9     & 28.4 & 32.7 & \lpld \stdb{32.8}{0.2} & 16.0     & 21.9 & 28.6 & \lpld \stdb{30.2}{0.8} & \lpqld\stdb{32.2}{0.2} & 11.4     & 13.2 & 20.2 & \lpld\stdb{23.1}{0.5} & \lpqld\stdb{29.6}{0.2} & \lpqld\stdb{27.3}{0.6} & \lpqld\stdb{20.0}{0.3} \\
    $\{\mathbf{z},\mathcal{F}\}$ & \multicolumn{3}{c|}{5,862 M} & \multicolumn{3}{c}{580 M} & \lpld 570 M & \multicolumn{3}{c}{310 M} & \lpld 310 M & \lpqld 300 M & \multicolumn{3}{c}{145 M} & \lpld 145 M & \lpqld 129 M & \lpqld 74 M & \lpqld 29 M\\
    $\{\mathbf{z},\mathcal{F}\}_{\uparrow}$ & \multicolumn{3}{c|}{1$\times$} & \multicolumn{3}{c}{10$\times$} & \lpld 10$\times$ & \multicolumn{3}{c}{19$\times$} & \lpld 19$\times$ & \lpqld 20$\times$ & \multicolumn{3}{c}{40$\times$} & \lpld 40$\times$ & \lpqld 45$\times$ & \lpqld 79$\times$ & \lpqld 202$\times$ \\
    \midrule
    \multicolumn{19}{l}{\textbf{IPC20} (331 M)} \\ \midrule
    Acc (\%)     & 33.6     & 44.0 & 47.2 & 31.1     & 39.7 & 44.7 & \lpld \stdb{45.0}{0.1} & 29.2     & 34.1 & 41.0 & \lpld \stdb{42.3}{0.2} & \lpqld\stdb{43.2}{0.3} & 21.7     & 24.0 & 33.0 & \lpld\stdb{35.7}{0.4} & \lpqld\stdb{41.2}{0.2} & \lpqld\stdb{38.6}{0.2} & \lpqld\stdb{30.5}{0.5} \\
    $\{\mathbf{z},\mathcal{F}\}$ & \multicolumn{3}{c|}{11,722 M} & \multicolumn{3}{c}{1167 M} & \lpld 1140 M & \multicolumn{3}{c}{623 M} & \lpld 623 M & \lpqld 604 M & \multicolumn{3}{c}{292 M} & \lpld 292 M & \lpqld 259 M & \lpqld 150 M & \lpqld 59 M \\
    $\{\mathbf{z},\mathcal{F}\}_{\uparrow}$ & \multicolumn{3}{c|}{1$\times$} & \multicolumn{3}{c}{10$\times$} & \lpld 10$\times$ & \multicolumn{3}{c}{19$\times$} & \lpld 19$\times$ & \lpqld 19$\times$ & \multicolumn{3}{c}{40$\times$} & \lpld 40$\times$ & \lpqld 45$\times$ & \lpqld 78$\times$ & \lpqld 199$\times$ \\
    \midrule
    \multicolumn{19}{l}{\textbf{IPC50} (819 M)} \\ \midrule
    Acc (\%)     & 46.8$^\dag$     & 53.5$^\dag$ & 55.4 & 44.1     & 50.3 & 54.4 & \lpld \stdb{54.5}{0.2} & 41.5     & 46.1 & 51.8 & \lpld \stdb{52.8}{0.1} & \lpqld\stdb{53.1}{0.1} & 35.5     & 38.0 & 46.7  & \lpld\stdb{48.4}{0.0} & \lpqld\stdb{51.3}{0.1} & \lpqld\stdb{49.6}{0.1} & \lpqld\stdb{43.0}{0.1} \\
    $\{\mathbf{z},\mathcal{F}\}$ & \multicolumn{3}{c|}{29,302 M} & \multicolumn{3}{c}{2,928 M} & \lpld 2,850 M & \multicolumn{3}{c}{1,562 M} & \lpld 1,562 M & \lpqld 1,515 M & \multicolumn{3}{c}{733 M} & \lpld 733 M & \lpqld 649 M & \lpqld 375 M & \lpqld 147 M \\
    $\{\mathbf{z},\mathcal{F}\}_{\uparrow}$ & \multicolumn{3}{c|}{1$\times$} & \multicolumn{3}{c}{10$\times$} & \lpld 10$\times$ & \multicolumn{3}{c}{19$\times$} & \lpld 19$\times$ & \lpqld 19$\times$ & \multicolumn{3}{c}{40$\times$} & \lpld 40$\times$ & \lpqld 45$\times$ & \lpqld 78$\times$ & \lpqld 199$\times$ \\
    \midrule
    \multicolumn{19}{l}{\textbf{IPC100} (1,630 M)} \\ \midrule
    Acc (\%)    & 52.8$^\dag$     & 58.0$^\dag$ & 59.4 & 51.1     & 55.1 & 58.8 & \lpld \stdb{59.1}{0.1} & 49.5     & 53.3 & 57.4 & \lpld \stdb{58.0}{0.2} & \lpqld\stdb{57.5}{0.1} & 44.4     & 47.2 & 54.0  & \lpld\stdb{54.9}{0.1} & \lpqld\stdb{56.2}{0.1} & \lpqld\stdb{54.9}{0.2} & \lpqld\stdb{50.1}{0.4} \\
    $\{\mathbf{z},\mathcal{F}\}$ & \multicolumn{3}{c|}{58,606 M} & \multicolumn{3}{c}{5,871 M} & \lpld 5,700 M & \multicolumn{3}{c}{3,132 M} & \lpld 3,132 M & \lpqld 3,038 M & \multicolumn{3}{c}{1,468 M} & \lpld 1,468 M & \lpqld 1,301 M & \lpqld 752 M & \lpqld 295 M \\
    $\{\mathbf{z},\mathcal{F}\}_{\uparrow}$ & \multicolumn{3}{c|}{1$\times$} & \multicolumn{3}{c}{10$\times$} & \lpld 10$\times$ & \multicolumn{3}{c}{19$\times$} & \lpld 19$\times$ & \lpqld 19$\times$ & \multicolumn{3}{c}{40$\times$} & \lpld 40$\times$ & \lpqld 45$\times$ & \lpqld 78$\times$ & \lpqld 199$\times$ \\
    \midrule
    \multicolumn{19}{l}{\textbf{IPC200} (3,256 M)} \\ \midrule
    Acc (\%)    & 57.0$^\dag$     & \textbf{63.3}$^\dag$ & 62.6 & 56.5     & 59.4 & 62.4 & \lpld \stdb{62.4}{0.1} & 55.1     & 58.3 & 61.7 & \lpld \stdb{61.9}{0.0} & \lpqld\stdb{60.8}{0.2} & 51.9     & 54.4 & 59.6  & \lpld\stdb{60.1}{0.1} & \lpqld\stdb{59.9}{0.2} & \lpqld\stdb{58.8}{0.2} & \lpqld\stdb{55.4}{0.2} \\
    $\{\mathbf{z},\mathcal{F}\}$ & \multicolumn{3}{c|}{117,208 M} & \multicolumn{3}{c}{11,753 M} & \lpld 11,400 M & \multicolumn{3}{c}{6,269 M} & \lpld 6,269 M & \lpqld 6,082 M & \multicolumn{3}{c}{2,939 M} & \lpld 2,939 M & \lpqld 2,606 M & \lpqld 1,509 M & \lpqld 594 M \\
    $\{\mathbf{z},\mathcal{F}\}_{\uparrow}$ & \multicolumn{3}{c|}{1$\times$} & \multicolumn{3}{c}{10$\times$} & \lpld 10$\times$ & \multicolumn{3}{c}{19$\times$} & \lpld 19$\times$ & \lpqld 19$\times$ & \multicolumn{3}{c}{40$\times$} & \lpld 40$\times$ & \lpqld 45$\times$ & \lpqld 78$\times$ & \lpqld 197$\times$ \\
    \bottomrule
\end{tabular}
\end{table*}

\subsection{Experiment Settings}
Dataset details and training settings are provided in Appendix~\ref{appendix:dataset}.
Computing resources used for experiments can be found in Appendix~\ref{appendix:resources}.

\textbf{Dataset.}
Our experiment results are evaluated on ImageNet-1K~\cite{deng2009imagenet}, and ImageNet-21K-P~\cite{ridnik2021imagenet}.
We follow the data pre-processing procedure of SRe$^2$L~\cite{yin2023squeeze} and CDA~\cite{yin2023dataset}.

\textbf{Squeeze.}
We modify the pretrained model by adding class-wise BN running mean and running variance; since they are not involved in computing the BN statistics, they do not affect performance.
As mentioned in Sec.~\ref{sec:diverse-sample}, we compute class-wise BN statistics by training for one epoch with model parameters kept frozen.

\textbf{Recover.}
We perform data synthesis following Eq.~\ref{eq:ours}.
The batch size for the recovery phase is the same as the IPC.
Besides, we adhere to the original settings of SRe$^2$L while changing the pretrained teacher model to our squeezed models with class-wise BN statistics.

\textbf{Relabel.}
We use pretrained ResNet18~\cite{he2016deep} for all experiments as the relabel model except otherwise stated.
For Tiny-ImageNet and ImageNet-1K, we use the PyTorch pretrained model.
For ImageNet-21K-P, we use the timm pretrained model.
For DELT~\cite{shen2024deltsimplediversitydrivenearlylate}, due to the missing implementation of the relabeling process, we implement the Random Augmentation~\cite{cubuk2019randaugment} under Timm's~\cite{rw2019timm} implementation.

\textbf{Label Pruning Setting.}
For label pruning, we exclude the last batch (usually with an incomplete batch size) of each epoch from the label pool.
To not introduce any advantages from discarding the incomplete batch, we reduce the training of one batch from every epoch.
To attain a $p$\% pruning ratio, we generate and store the first $(1-p)$\% of total batches.
To implement label pruning effectively, we store additional image indices for every batch in a single file.
The additional storage is accounted for within the total auxiliary information $\{\mathcal{F}\}$.

\textbf{Label Quantization Setting.}
We store both indices and values of the Top-$k$ output logits.
During reconstruction of the full output logits, we set non-Top-$k$ values to zero.

\textbf{Validation.}
For LPLD and LPQLD, we adhere to the hyperparameter settings of CDA~\cite{yin2023dataset}.
Unless otherwise stated, we follow their original hyperparameter settings for all methods.

\begin{table*}
\centering
\caption{
Label pruning result on ImageNet-21K-P, using ResNet-18.
$\mathbb{I}$ denotes image storage.
$\{\mathbf{z}, \mathcal{F}\}$ denotes total storage of both augmentation information and soft labels.
Setting A, B, C denote using different base label pruning rates of 10$\times$, 20$\times$, and 40$\times$, respectively.
Best setting is highlighted in {\color{myblue2}\textbf{blue bold}} text.
$^\dag$ denotes reported results.
}
\label{tab:in21k}
\setlength{\tabcolsep}{0.22em}
\begin{tabular}{@{}lc|ccccc|ccc|ccc|c|cc|cc|cc@{}}
\toprule
      &              & \multicolumn{5}{c|}{1$\times$}        & \multicolumn{3}{c|}{10$\times$} & \multicolumn{3}{c|}{40$\times$} & Setting & \multicolumn{2}{c|}{$\sim$100$\times$} & \multicolumn{2}{c|}{$\sim$400$\times$} & \multicolumn{2}{c}{$\sim$500$\times$} \\
IPC   & $\mathbb{I}$ & $\{\mathbf{z}, \mathcal{F}\}$ & SRe$^2$L & CDA & D3S & LPLD & $\{\mathbf{z}, \mathcal{F}\}$       & LPLD  & LPQLD     & $\{\mathbf{z}, \mathcal{F}\}$       & LPLD & LPQLD & LPQLD & $\{\mathbf{z}, \mathcal{F}\}$       & LPQLD & $\{\mathbf{z}, \mathcal{F}\}$       & LPQLD  & $\{\mathbf{z}, \mathcal{F}\}$     & LPQLD  \\ \midrule
\multirow{3}{*}{IPC10} & \multirow{3}{*}{3G} & \multirow{3}{*}{643G} & \multirow{3}{*}{18.5$^\dag$} & \multirow{3}{*}{22.6$^\dag$} & \multirow{3}{*}{\textbf{26.9}$^\dag$} & \multirow{3}{*}{25.4} & \multirow{3}{*}{65G} & & & \multirow{3}{*}{16G} & & & 10$\times$ & 6.4G & \lpqld 23.5 & 1.6G & \lpqld 20.5 & 1.2G & \lpqld 18.9 \\
 &  &  &  &  &  &  &  & 24.1 & \lpld \textbf{28.2} &  & 21.3 & \lpld \textbf{25.6} & 20$\times$ & 6.6G & \lpqld \textbf{23.9} & 1.6G & \lpqld \textbf{20.9} & 0.8G & \lpqld \textbf{19.2} \\
 &  &  &  &  &  &  &  &  &  &  &  &  & 40$\times$ & \textbf{6.0G} & \lpqld 23.3 & \textbf{1.3G} & \lpqld 19.8 & \textbf{0.8G} & \lpqld 18.1 \\ \midrule
\multirow{3}{*}{IPC20} & \multirow{3}{*}{5G} & \multirow{3}{*}{1285G} & \multirow{3}{*}{20.5$^\dag$} & \multirow{3}{*}{26.4$^\dag$} & \multirow{3}{*}{28.5$^\dag$} & \multirow{3}{*}{\textbf{30.3}} & \multirow{3}{*}{129G} & & & \multirow{3}{*}{32G} & & & 10$\times$ & 13G & \lpqld 26.1 & 3.1G & \lpqld 23.0 & 2.3G & \lpqld 21.9 \\ 
 &  &  &  &  &  &  &  &  31.3 &  \lpld \textbf{34.4} &     & 29.4 & \lpld \textbf{33.8} & 20$\times$ & 14G & \lpqld 27.1 & 3.3G & \lpqld 23.9 & 1.7G & \lpqld 22.7 \\ 
 &  &  &  &  &  &  &  &  &  & &  &            & {\color{myblue2}\textbf{40$\times$}} & \textbf{12G} & \lpqld \textbf{27.7} & \textbf{2.7G} & \lpqld \textbf{24.1} & \textbf{1.6G} & \lpqld \textbf{23.3} \\ \bottomrule
\end{tabular}
\end{table*}

\subsection{Primary Results}

\textbf{ImageNet-1K.}
Table~\ref{tab:label_pruning_main} compares the ImageNet-1K pruning results with SOTA methods on ResNet18.
Our method outperforms other SOTA methods at various pruning ratios and different IPCs.
More importantly, our method consistently exceeds the unpruned version of SRe$^2$L with 78$\times$ less storage.
This may not appear impressive at first glance; however, in terms of actual storage, the size is reduced from 29.7G to 0.38G for IPC50.
In addition, we notice the performance at 10$\times$ (or 90\%) pruning ratio degrades slightly, especially for large IPCs.
For example, there is merely a $0.2\%$ performance degradation on IPC200 using ResNet18.
Based on this observation, we apply quantization at a 10$\times$ pruning rate, which means to achieve $\sim$80$\times$ storage reduction, we take the sweet spot of 10$\times$ from pruning, and reduce the output logits by 8$\times$.
Following this approach, we notice a huge improvement over solely applying pruning.
For instance, 7.2\% (27.3\% vs. 20.1\%) accuracy improvements on IPC10 at $79\times$ label compression rate.
The improvement is consistent and is more prominent at small IPCs.

\textbf{ImageNet-21K-P.}
ImageNet-21K-P has 10,450 classes, significantly increasing the disk storage as each soft label stores probabilities for 10,450 classes.
Compared to ImageNet-1K, ImageNet-21K has a larger label-to-image ratio due to a larger number of classes. 
The IPC20 dataset leads to a 1.2 TB (i.e., 1285 GiB) label storage, making the existing framework less practical.
However, with the help of our method, it can surpass SRe$^2$L~\cite{yin2023squeeze} using 500$\times$ less label storage.

In addition, exhibited in ImageNet-1K (Table~\ref{tab:label_pruning_main}), the Pareto fronts for optimal pruning and quantization ratio combinations change with the dataset size.
With a larger IPC, label pruning delivers an improved performance. 
Such a behavior is more distinguished in ImageNet-21K (Table~\ref{tab:in21k}), and we have set 3 different settings at high compression rates. 
For IPC20, using a base pruning rate of 40$\times$ outperforms smaller base pruning rates in terms of both performance and storage.

\begin{table}[t]
\setlength{\tabcolsep}{2.6pt} 
\centering
\caption{Storage breakdown of soft labels of a single batch of size 128. \texttt{Q} denotes label quantization and \texttt{P} denotes label pruning. Note that the bytes here include the overhead for PyTorch's serialization format. ImageNet-1K.}
\label{tab: label-storage-breakdown}
\begin{tabular}{l c c c c c c}
\toprule
Component                                 & Shape  & Bytes   & \% of Total      & Type    & Q          & P \\ 
\midrule
1.\ coordinates of crops                  & [128, 4]       & 2224    & 0.86\%  & $\mathcal{F}$    & $\times$     & $\checkmark$ \\
2.\ flip status                           & [128]          & 1328    & 0.51\%  & $\mathcal{F}$    & $\times$     & $\checkmark$ \\
3.\ index of cutmix images                & [128]          & 1328    & 0.51\%  & $\mathcal{F}$    & $\times$     & $\checkmark$ \\
4.\ strength of cutmix                    & [1]            & 880     & 0.34\%  & $\mathcal{F}$    & $\times$     & $\checkmark$ \\
5.\ cutmix bounding box                   & [4]            & 1904    & 0.73\%  & $\mathcal{F}$    & $\times$     & $\checkmark$ \\
6.\ prediction logits                     & [128, 1000]    & 257200  & \textbf{98.94}\% & $\mathbf{z}$    & $\checkmark$ & $\checkmark$ \\
\bottomrule
\end{tabular}
\end{table}

\textbf{Label Storage Breakdown.}
Table~\ref{tab: label-storage-breakdown} shows a storage gap between the theoretical and actual compression rates, highlighting the roles of label quantization and pruning. The theoretical rate only considers the output logits $\mathbf{\{z\}}$, while the actual rate must also include auxiliary data, particularly the augmentation information $\mathcal{\{F\}}$.
For label quantization, even though augmentation data occupies only up to 1.06\% of total storage, its impact grows with higher compression. For example, at IPC10, the theoretical 500$\times$ compression rate is reduced to an actual compression rate of 198$\times$, since $\mathcal{\{F\}}$ uses 20MB of a 30MB storage budget.
Label pruning removes both $\mathbf{\{z\}}$ and $\mathcal{\{F\}}$, but the actual compression rate is slightly lower than the theoretical rate since we store image indices of every batch, trading a small overhead for implementation flexibility.

\subsection{Ablation Study}

\begin{table}[t]
\centering
\caption{Ablation study of LPQLD. 
$\square$ denotes image generation components, $\triangle$ denotes soft label compression components.
\texttt{C}: Class-wise matching. 
\texttt{CS}: Class-wise supervision. 
\texttt{BP}: Batch-level pruning. 
\texttt{DKR}: Dynamic knowledge reuse. 
\texttt{Q}: Label quantization (Top-$k$). 
\texttt{CA}: Calibrated student-teacher alignment.
Results are on ResNet-18, ImageNet-1K, IPC50.}
\label{tab:ablation_unified}
\setlength{\tabcolsep}{3pt}
\renewcommand{\arraystretch}{1.05}
\begin{tabular}{@{}c|*{6}{@{\hspace{3pt}}c@{\hspace{3pt}}}|*{3}{@{\hspace{4pt}}l@{\hspace{4pt}}}|@{\hspace{5pt}}l@{}}
\toprule
 & \texttt{C} & \texttt{CS} & \texttt{BP} & \texttt{DKR} & \texttt{Q} & \texttt{CA} & 10$\times$ & 50$\times$ & 100$\times$ & Full \\ \midrule
$\square$ & -              & -              & -              & -              & -             & -              & 49.4 & 34.8 & 25.4 & 52.0 \\
$\square$ & $\checkmark$   & -              & -              & -              & -             & -              & \up{51.9}{2.5} & \up{37.7}{2.9} & \down{22.6}{2.8} & \up{54.7}{2.7} \\
$\square$ & $\checkmark$   & $\checkmark$   & -              & -              & -             & -              & \up{53.2}{3.8} & \up{39.7}{4.9} & \up{29.1}{3.7} & \up{55.4}{3.4} \\
$\triangle$ & $\checkmark$   & $\checkmark$   & $\checkmark$   & -              & -             & -              & \up{54.4}{5.0} & \up{43.1}{8.3} & \up{33.7}{8.3} & \up{55.4}{3.4} \\
$\triangle$ & $\checkmark$   & $\checkmark$   & $\checkmark$   & $\checkmark$   & -             & -              & \up{54.4}{5.0} & \up{47.8}{13.0} & \up{41.4}{16.0} & \up{55.4}{3.4} \\
$\triangle$ & $\checkmark$   & $\checkmark$   & $\checkmark$   & $\checkmark$   & $\checkmark$  & -              & - & \up{49.7}{14.9} & \up{47.9}{22.5} & \up{55.4}{3.4} \\
$\triangle$ & $\checkmark$   & $\checkmark$   & $\checkmark$   & $\checkmark$   & $\checkmark$  & $\checkmark$   & - & \up{51.3}{16.5} & \up{49.6}{24.2} & \up{55.4}{3.4} \\ \bottomrule
\end{tabular}
\end{table}
Table~\ref{tab:ablation_unified} presents a comprehensive ablation study that incrementally adds each component to isolate their individual effects and demonstrate their complementary benefits. We analyze the results across different compression rates (10$\times$, 50$\times$, 100$\times$) and full label settings. Rows are organized into two groups: \textbf{$\square$ Image Generation} (Rows 1-3) focuses on improving synthetic image diversity, while \textbf{$\triangle$ Label Compression} (Rows 4-7) addresses soft label storage efficiency.

$\square$ \textbf{Row 1 (Baseline):} SRe$^2$L implementation under CDA's hyperparameter settings. Label compression is achieved by random epoch-level label pruning.

$\square$ \textbf{Row 2 (+ Class-wise matching):} Re-batches images within the same class during image ``recover'' phase. This modification increases within-class diversity by exposing the model to varied intra-class patterns.

$\square$ \textbf{Row 3 (+ Class-wise supervision):} Replaces global BN statistics with class-wise BN statistics computed separately for each class in the ``squeeze'' phase, which are then used as supervision in the ``recover'' phase. Since Row 2 already uses class-wise matching, this provides more precise and reasonable supervision that aligns with the class-specific batching, leading to stronger diversity constraints and larger improvements at higher compression rates.

$\triangle$ \textbf{Row 4 (+ Batch-level pruning):} Switches from epoch-level to batch-level label pruning, enabling finer-grained control over which augmentation-label pairs to retain. This structured selection mechanism consistently improves performance across all compression rates.

$\triangle$ \textbf{Row 5 (+ Dynamic knowledge reuse):} Introduces temperature annealing that gradually decreases teacher temperature during training. This allows the model to extract maximum information from retained labels by starting with softer distributions (learning broad class relationships) and transitioning to sharper distributions (refining specific class boundaries). The gains are most pronounced at high compression rates where label scarcity makes efficient knowledge extraction critical.

$\triangle$ \textbf{Row 6 (+ Label quantization):} Applies Top-$k$ selection to compress soft labels by storing only the top-$k$ \textit{pre-softmax} logits. These stored logits can be later recomputed with different temperatures during training, enabling flexible temperature scheduling. Combined with label pruning, this achieves additional compression. Quantization is not applied at 10$\times$ since the base pruning already operates at this rate.

$\triangle$ \textbf{Row 7 (+ Calibrated alignment):} Adds grid-search optimization to dynamically adjust student temperature, aligning the student distribution with the sharper quantized teacher distribution. This calibration is essential for teacher label quantization, without it (as would occur in Row 5), temperature adjustment creates distribution mismatch. Here, it properly compensates for the information loss from quantization, enabling effective dual compression.

\subsection{Additional Analysis}
\label{sec: additional-analysis}

\begin{table}[t]
\footnotesize
\caption{LPQLD on different distilled datasets. 
$^\dag$ denotes the reported results. ResNet-18, ImageNet-1K.}
\label{tab: cross-method}
\setlength{\tabcolsep}{0.18em}
\begin{tabular}{@{}l|cccc|cccc|cccc@{}}
\toprule
    & \multicolumn{4}{c|}{DWA~\cite{dwa2024neurips}} & \multicolumn{4}{c|}{DELT~\cite{shen2024deltsimplediversitydrivenearlylate}} & \multicolumn{4}{c}{Minimax~\cite{gu2024efficient}} \\ 
$\mathbf{\{z\}}$    & 1$\times$ & 10$\times$    & 50$\times$    & 100$\times$   & 1$\times$  & 10$\times$        & 50$\times$        & 100$\times$   & 1$\times$     & 10$\times$    & 50$\times$        & 100$\times$     \\\midrule
10  &  37.9$^\dag$ & \lpld 33.4  & \lpqld 29.2   & \lpqld 26.4  & 45.8$^\dag$ &  \lpld 41.3 & \lpqld 35.6   & \lpqld 35.0   & 44.3$^\dag$  & \lpld 39.3      & \lpqld 37.7       & \lpqld 34.9       \\
50  &  55.2$^\dag$ & \lpld 51.7  & \lpqld 47.8   & \lpqld 46.0  & 59.2$^\dag$ &  \lpld 58.0 & \lpqld 55.5   & \lpqld 54.4   & 58.6$^\dag$  & \lpld 56.1      & \lpqld 55.0       & \lpqld 53.7       \\
100 &  59.2$^\dag$ & \lpld 57.3  & \lpqld 54.6   & \lpqld 52.9  & 62.4$^\dag$ &  \lpld 61.7 & \lpqld 58.6   & \lpqld 58.5   & -            & \lpld 59.0      & \lpqld 58.4       & \lpqld 57.4       \\ \bottomrule
\end{tabular}
\end{table}

\begin{table}[t]
\footnotesize
\centering
\caption{
Cross-architecture result. ResNet-18 is used as the teacher model for all student models.
IPC50, ImageNet-1K.
}
\label{tab: cross_arch}
\setlength{\tabcolsep}{0.56em}
\begin{tabular}{@{}lrc|cccc@{}}
\toprule
 Model           & Size \space\space        & Full Acc     & 1$\times$         & $\sim$30$\times$  & $\sim$40$\times$  & $\sim$80$\times$ \\ \midrule

 ResNet-18~\cite{he2016deep}       & 11.7M                 & 69.76 & 55.4 & \lpld 50.2 & \lpqld 51.3 & \lpqld 49.6 \\
 ResNet-50~\cite{he2016deep}       & 25.6M                 & 76.13 & 62.2 & \lpld 57.8 & \lpqld 56.7 & \lpqld 54.8 \\
 EfficientNet-B0~\cite{tan2019efficientnet} & 5.3M         & 77.69 & 55.5 & \lpld 53.2 & \lpqld 54.0 & \lpqld 52.1 \\
 MobileNet-V2~\cite{sandler2018mobilenetv2}    & 3.5M      & 71.88 & 49.1 & \lpld 47.8 & \lpqld 45.9 & \lpqld 45.8 \\
 Swin-V2-Tiny~\cite{liu2022swin}    & 28.4M                & 82.07 & 40.6 & \lpld 42.0 & \lpqld 44.6 & \lpqld 38.7 \\ \bottomrule
\end{tabular}
\end{table}

\begin{table}[t]
\centering
\caption{Comparison with prior label quantization methods, MS (Marginal Smoothing) and MR (Marginal Renorm)~\cite{shen2022fast}.
$^*$ denotes MS and MR having approximately the same storage.
ImageNet-1K, ResNet-18.}
\footnotesize
\label{tab:compare_quantization}
\setlength{\tabcolsep}{2.5pt}
\renewcommand{\arraystretch}{1.1}
\begin{tabular}{l|cccc|cccc}
\toprule
Method & \multicolumn{4}{c|}{IPC10} & \multicolumn{4}{c}{IPC50} \\
Top-$k$ & $k$=50 & $k$=25 & $k$=5 & $k$=1 & $k$=50 & $k$=25 & $k$=5 & $k$=1 \\ 
\midrule
MS~\cite{shen2022fast} & 16.5 & 12.4 & 7.5 & 4.0 & 38.8 & 33.6 & 19.5 & 7.7 \\[-0.5ex]
 & {\color{gray}$\scriptscriptstyle \pm 0.2$} & {\color{gray}$\scriptscriptstyle \pm 0.4$} & {\color{gray}$\scriptscriptstyle \pm 0.2$} & {\color{gray}$\scriptscriptstyle \pm 0.1$} & {\color{gray}$\scriptscriptstyle \pm 0.3$} & {\color{gray}$\scriptscriptstyle \pm 0.2$} & {\color{gray}$\scriptscriptstyle \pm 0.3$} & {\color{gray}$\scriptscriptstyle \pm 0.6$} \\
MR~\cite{shen2022fast} & 17.7 & 16.6 & 12.6 & 5.1 & 28.6 & 31.7 & 34.0 & 16.7 \\[-0.5ex]
 & {\color{gray}$\scriptscriptstyle \pm 0.4$} & {\color{gray}$\scriptscriptstyle \pm 0.2$} & {\color{gray}$\scriptscriptstyle \pm 0.4$} & {\color{gray}$\scriptscriptstyle \pm 0.1$} & {\color{gray}$\scriptscriptstyle \pm 0.2$} & {\color{gray}$\scriptscriptstyle \pm 0.2$} & {\color{gray}$\scriptscriptstyle \pm 0.2$} & {\color{gray}$\scriptscriptstyle \pm 0.2$} \\
\midrule
$\{\mathbf{z},\mathcal{F}\}^*$ & 753M & 478M & 203M & 203M & 3,768M & 2,393M & 1,020M & 1,020M \\
\midrule
\lpqld \textbf{Ours} & \multicolumn{4}{c|}{\lpqld \std{\textbf{29.6}}{0.2}} & \multicolumn{4}{c}{\lpqld \std{\textbf{51.3}}{0.1}} \\
\lpqld $\{\mathbf{z},\mathcal{F}\}$ & \multicolumn{4}{c|}{\lpqld \textbf{129M}} & \multicolumn{4}{c}{\lpqld \textbf{649M}} \\
\bottomrule
\end{tabular}
\end{table}

\textbf{Cross-Architecture Performance.}
Table~\ref{tab: cross_arch} demonstrates how our compression methods generalize across diverse network architectures, with ResNet-18 serving as the teacher model throughout. 
The results reveal consistent patterns across different model families, from convolutional networks (ResNet-18, ResNet-50) to efficient mobile architectures (MobileNet-V2, EfficientNet-B0) and transformer-based models (Swin-V2-Tiny).
At 10$\times$ compression using LPLD, all architectures maintain performance nearly equivalent to their uncompressed (1$\times$) counterparts, with minimal accuracy drops ranging from 0 to 2 percentage points.
More aggressive compression (30$\times$, 50$\times$, and 100$\times$) shows expected performance degradation, though the hybrid LPQLD method helps mitigate losses at higher compression rates.
Notably, the Swin-V2-Tiny transformer model shows unique behavior, actually \textbf{improving} by 2.2 percentage points at 10$\times$ compression and maintaining strong performance at 50$\times$ compression with LPQLD.
These cross-architecture results highlight our findings that when dataset diversity is sufficient, the need for extensive augmentations and soft labels can be significantly reduced regardless of the target architecture.

\textbf{Label Compression Performance on Different Distilled Datasets.}
Since our label pruning and quantization techniques are explicit and transferable, we evaluate the effectiveness across various distilled datasets in Table~\ref{tab: cross-method}. We apply our label pruning approach to three additional dataset distillation methods: the diversity-driven DWA~\cite{dwa2024neurips}, multisized-pattern DELT~\cite{shen2024deltsimplediversitydrivenearlylate}, and diffusion-based Minimax~\cite{gu2024efficient}. Results demonstrate strong compatibility with all methods, with particularly significant improvements observed for IPC settings.

\textbf{Comparison with Prior Label Quantization Methods.}
Table~\ref{tab:compare_quantization} compares LPQLD with FKD's Marginal Smoothing (MS) and Marginal Re-Norm (MR)~\cite{shen2022fast}. Both MS and MR suffer severe performance degradation as Top-$k$ decreases, with \textbf{Top-5 and Top-1 having identical storage} (203M for IPC10, 1,020M for IPC50) because auxiliary data dominates, representing the \textbf{fundamental limit of quantization-only approaches}. In contrast, LPQLD achieves superior performance (29.6\% for IPC10, 51.3\% for IPC50) with significantly lower storage (129M and 649M) by combining pruning and quantization, breaking through this bottleneck.

\begin{figure}[t]
    \centering
    \includegraphics[width=.65\linewidth]{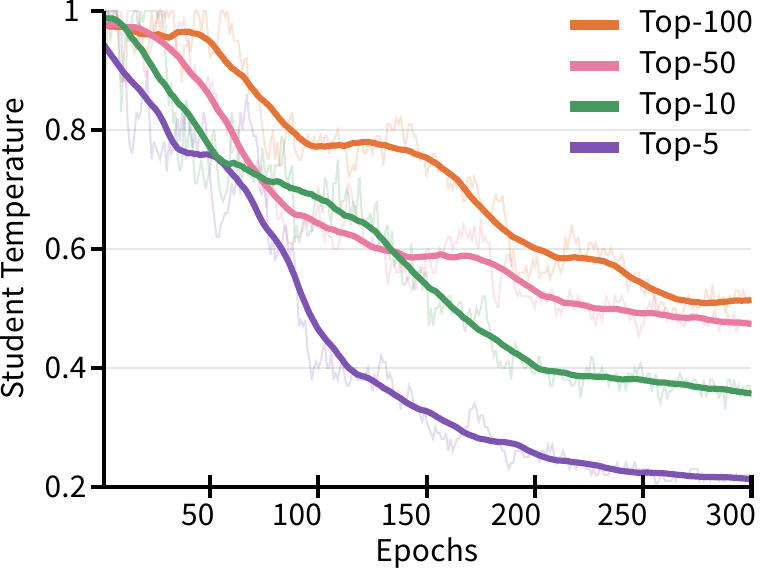}
    \caption{
    Optimal student temperature $\hat{\tau}^*$ over training epochs.
    }
    \label{fig: student-temperature}
\end{figure}

\begin{figure}[t]
    \centering
    \includegraphics[width=.75\linewidth]{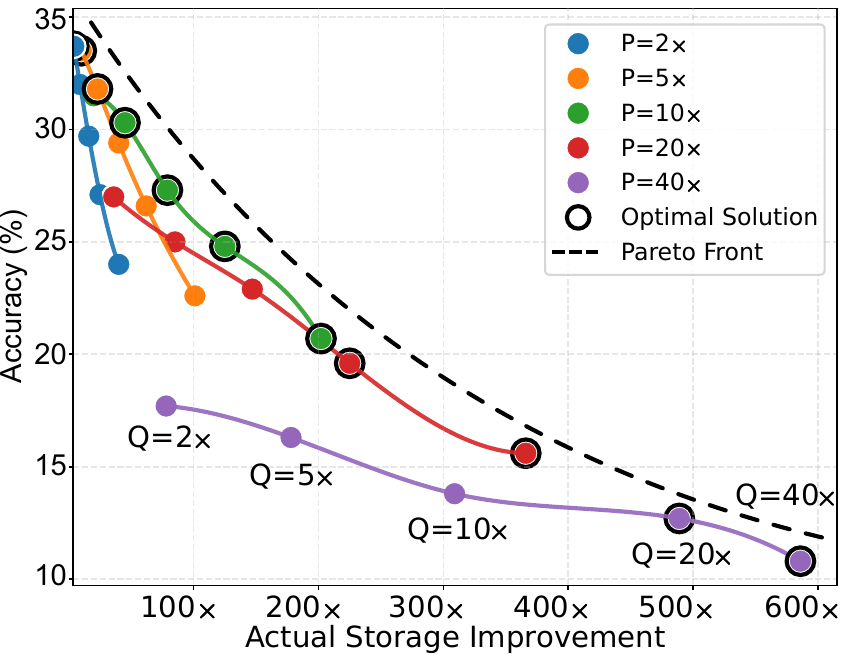}
    \caption{
    Pareto front of accuracy vs. storage for different label pruning (P) and label quantization (Q) combinations. IPC10, ResNet-18, ImageNet-1K.
    }
    \label{fig: pareto-front}
    \vspace{-1em}
\end{figure}

\textbf{Optimal Student Temperature.}
Fig.~\ref{fig: student-temperature} showcases when the student model is randomly initialized and predictions are not accurate, a relatively high temperature is assigned. When models can make more accurate predictions, the optimal temperature for students goes down to better align with the concentrated probability mass of the teacher distribution.
In addition, we observe that a higher quantization ratio (smaller Top-$k$) has a lower average student temperature,
which is consistent with our hypothesis that student prediction distribution should be scaled more to properly align with a more concentrated teacher distribution.

\textbf{Pareto Front Analysis in a Fixed IPC.}
From Fig.~\ref{fig: pareto-front}, we present two critical observations.
First, there exists a Pareto front where specific configurations simultaneously excel in both accuracy and storage efficiency compared to alternatives.
Second, even within a single IPC (i.e., IPC10), we observe significant configuration variations along the Pareto front.
Specifically, when actual storage improvements ($\{\mathbf{z}, \mathcal{F}\}$) remain below $50\times$, combining a $5\times$ label pruning rate with label quantization outperforms configurations using larger pruning ratios as the base compression rate. However, as the actual compression rate increases to $50\times$-$200\times$, configurations with a $10\times$ label pruning rate clearly dominate the Pareto front. Beyond $200\times$ compression, adopting a $20\times$ label pruning rate emerges as the optimal approach.

\textbf{Pareto Front Analysis in Different IPCs.}
Fig.~\ref{fig: pareto-ipc} demonstrates that \textbf{larger dataset sizes favor higher base compression rates at equivalent compression ratios}.
Given a fixed theoretical compression ratio $\{\mathbf{z}\}$, we observe that Pareto optimal solutions shift across dataset scales. At IPC10, configurations with a $10\times$ pruning rate (green dot) achieve higher performance despite lower actual storage improvement compared to configurations using a $20\times$ pruning rate as base. However, as the dataset scales from IPC10 to IPC200, this performance gap consistently diminishes. Eventually, at IPC200, the $20\times$ configuration (red dot) achieves dominance by maintaining comparable performance levels while delivering higher actual storage improvements.

\textbf{Visualization.}
Fig.~\ref{fig: viz-ours} visualizes our method on three classes.
More visualizations are provided in Appendix~\ref{appendix:viz}.

\begin{figure}[t]
    \centering
    \includegraphics[width=0.9\linewidth]{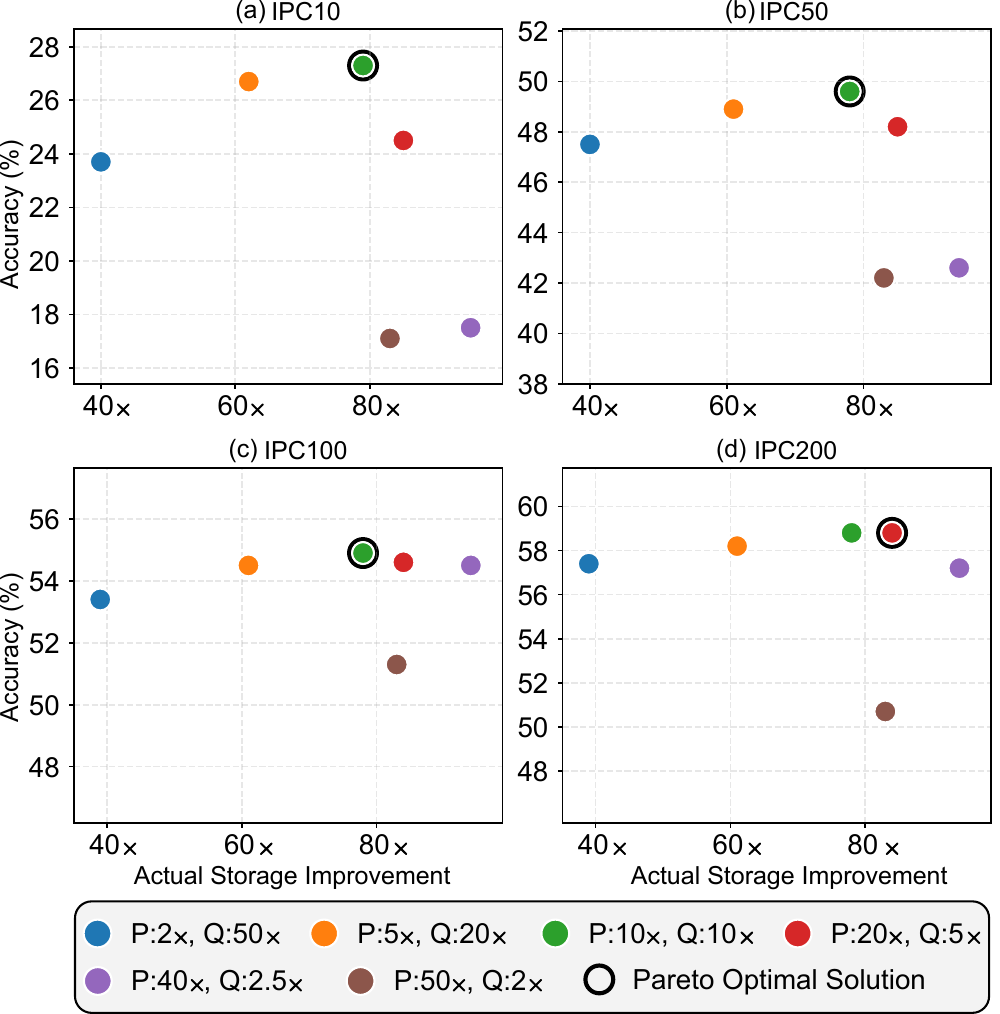}
    \caption{
    Illustration of Pareto front changes for different IPCs under the same theoretical compression rate (i.e., 100$\times$). \texttt{P:50x, Q:2x} denotes theoretically compression of label pruning is 50$\times$ and that of label quantization is 2$\times$.
    }
    \label{fig: pareto-ipc}
    \vspace{-1em}
\end{figure}

\section{Conclusion}

This work addressed the critical storage bottleneck in dataset distillation by developing Label Pruning and Quantization for Large-scale Distillation (LPQLD). Our analysis identified two fundamental issues: (1) \textit{insufficient image diversity}, where high within-class similarity requires extensive augmentation, and (2) \textit{insufficient supervision diversity}, where limited variety in supervisory signals leads to performance degradation at high compression rates. We systematically addressed both challenges by enhancing \textit{image diversity} via class-wise BN supervision during synthesis, and introducing Label Pruning with Dynamic Knowledge Reuse for \textit{label-per-augmentation diversity} and Label Quantization with Calibrated Student-Teacher Alignment for \textit{augmentation-per-image diversity}. Our extensive experiments demonstrate that LPQLD achieves 78$\times$ compression on ImageNet-1K and 500$\times$ on ImageNet-21K while maintaining or exceeding prior performance. By introducing the supervision diversity framework, we enable practical deployment in resource-constrained environments and pave the path for soft-label empowered dataset distillation.
Limitations and future work are in Appendix~\ref{appendix:limitation}. Ethical considerations are in Appendix~\ref{appendix:ethics}.

\normalem
\bibliographystyle{unsrt}
\bibliography{reference}

\clearpage
\onecolumn
\appendices

\section{Theoretical Motivation for Low Student Temperature}
\label{appendix:kl-ratio-proof}

\subsection{Optimal Logit-Temperature Relationship}

Let $\mathcal{P}$ be the quantized teacher distribution where $\mathcal{P}_i > 0$ for $i \in \mathcal{K}$ (top-$k$ indices) and $\mathcal{P}_i = 0$ otherwise, with $\sum_{i \in \mathcal{K}} \mathcal{P}_i = 1$. The student distribution $\mathcal{\hat{P}}$ is parameterized by logits $\mathbf{\hat{z}}$ and temperature $\hat{\tau}$.

\begin{theorem}
\label{thm:main}
Minimizing $D_{\mathrm{KL}}(\mathcal{P} \| \mathcal{\hat{P}})$ implies that for any pair of classes $i, j \in \mathcal{K}$, the optimal student logits $\mathbf{\hat{z}}^*$ satisfy:
\begin{equation}
\hat{z}_i^* - \hat{z}_j^* = \hat{\tau} \log \frac{\mathcal{P}_i}{\mathcal{P}_j}.
\label{eq:logit-relationship}
\end{equation}
\end{theorem}

\begin{proof}
Since $\mathcal{P}_m = 0$ for $m \notin \mathcal{K}$, only terms with $i \in \mathcal{K}$ contribute to the KL objective. The gradient w.r.t.\ $\hat{z}_k$ is $\frac{\partial D_{\mathrm{KL}}}{\partial \hat{z}_k} = -\frac{1}{\hat{\tau}}(\mathcal{P}_k - \mathcal{\hat{P}}_k)$, and setting it to zero yields the optimality condition $\mathcal{\hat{P}}_k = \mathcal{P}_k$ for $k \in \mathcal{K}$. Note that since the student softmax operates over all $C$ classes, exact matching is unattainable; however, the probability ratios within $\mathcal{K}$ can be matched. Substituting $\mathcal{\hat{P}}_i \propto \exp(\hat{z}_i/\hat{\tau})$ into $\mathcal{\hat{P}}_i / \mathcal{\hat{P}}_j = \mathcal{P}_i / \mathcal{P}_j$ and taking the logarithm yields Eq.~\eqref{eq:logit-relationship}.
\end{proof}

\subsection{Temperature Constraints}

\begin{proposition}
\label{prop:temp-scaling}
Given a maximum feasible logit difference $\Delta z_{\max}$ imposed by numerical stability or regularization, matching a quantized distribution with probability ratio $r = \max_{i,j \in \mathcal{K}} (\mathcal{P}_i/\mathcal{P}_j)$ requires:
\begin{equation}
\hat{\tau} \leq \frac{\Delta z_{\max}}{\log r}.
\label{eq:temp-bound}
\end{equation}
\end{proposition}

\begin{proof}
From Eq.~\eqref{eq:logit-relationship}, matching probability ratio $r$ requires $|\hat{z}_i^* - \hat{z}_j^*| = \hat{\tau} \log r$. Under constraint $|\hat{z}_i^* - \hat{z}_j^*| \leq \Delta z_{\max}$, we obtain $\hat{\tau} \leq \Delta z_{\max} / \log r$.
\end{proof}

\begin{remark}
\textbf{Why $\hat{\tau} < 1$ helps:} Quantization (top-$k$ selection) produces a lower-entropy distribution by concentrating probability mass on fewer classes. To match such a sharp distribution, larger logit differences are required. Under the constraint $|\hat{z}_i^* - \hat{z}_j^*| \leq \Delta z_{\max}$, a lower temperature $\hat{\tau} < 1$ compensates by amplifying the effective logit differences. This theoretical insight justifies our adaptive grid search strategy (Algorithm~1), as validated in Fig.~10.
\end{remark}

\section{Experimental Setup}
\label{appendix:dataset}
\label{appendix:setting}
\label{appendix:compression}

\subsection{ImageNet-1K}

\textbf{Dataset.} ImageNet-1K~\cite{deng2009imagenet} (ILSVRC2012) contains 1,000 categories with 1,281,167 training images and 50,000 validation images. All images are resized to $224 \times 224$ pixels.

\textbf{Training Configuration.}
We employ a PyTorch pretrained ResNet-18~\cite{he2016deep} with Top-1 accuracy of 69.76\% for both recovery and relabeling phases.
Class-wise BN statistics are extracted through a modified training script (Table~\ref{tab:squeeze-in1k}).
Data synthesis follows Table~\ref{tab:recover-in1k}, based on CDA~\cite{yin2023dataset} with IPC-dependent batch size.
For relabeling and validation (Table~\ref{tab:validate-in1k}), we use temperature annealing: $\tau$ starts at 20 and decreases by $\times 0.7$ every 30 epochs until reaching 2.
\begin{table}[H]
\centering
\begin{minipage}{\appendixnarrowtablewidth}
\centering
\captionsetup{width=\linewidth}
\caption{Squeezing configurations of ImageNet-1K.}
\label{tab:squeeze-in1k}
\begin{tabularx}{\linewidth}{@{}lX@{}}
\toprule
Info          & Detail                                           \\ \midrule
Total Images  & 1,281,167                                        \\
Batch Size    & 256                                              \\
BN Updates    & 5005                                             \\
Source        & \href{https://github.com/pytorch/vision/tree/main/references/classification}{https://github.com/pytorch/vision/references/classification}       \\ \bottomrule
\end{tabularx}
\end{minipage}
\end{table}

\begin{table}[H]
\centering
\begin{minipage}{\appendixtablewidth}
\centering
\captionsetup{width=\linewidth}
\caption{Data Synthesis of ImageNet-1K.}
\label{tab:recover-in1k}
\begin{tabularx}{\linewidth}{@{}XXX@{}}
\toprule
Config             & Detail                         \\ \midrule
Iteration          & 4,000         & -                              \\
Optimizer          & Adam          & $\beta_1,\beta_2 = (0.5, 0.9)$ \\
Image LR           & 0.25          & -                              \\
Batch Size         & IPC-dependent & e.g., 50 for IPC50             \\
Initialization     & Random        & -                              \\
BN Loss $(\alpha)$ & 0.01          & -                              \\ \bottomrule
\end{tabularx}
\end{minipage}
\end{table}

\begin{table}[H]
\centering
\begin{minipage}{\appendixtablewidth}
\centering
\captionsetup{width=\linewidth}
\caption{Relabel and Validation of ImageNet-1K. $\tau$ denotes temperature.}
\label{tab:validate-in1k}
\begin{tabularx}{\linewidth}{@{}XXX@{}}
\toprule
Config                        & Value                & Detail                    \\ \midrule
Epochs                        & 300                  & -                         \\
Optimizer                     & AdamW                & -                         \\
Model LR                      & 0.001                & -                         \\
Batch Size                    & 128                  & -                         \\
LR Scheduler                  & CosineAnnealing      & -                         \\
EMA Rate                      & Not Used             & -                         \\
\multirow{3}{*}{Augmentation} & RandomResizedCrop    & scale ratio = (0.08, 1.0) \\
                              & RandomHorizontalFlip & probability = 0.5         \\
                              & CutMix               & -                         \\ \midrule
Initial $\tau$                & 20                   & -                         \\
Final $\tau$                  & 2                    & clip at 2                 \\
$\tau$ Scheduler              & Step                 & $\times 0.7$ every 30 epochs   \\
\bottomrule
\end{tabularx}
\end{minipage}
\end{table}

\textbf{Compression Settings.}
Table~\ref{tab:lpqld_in1k} provides a detailed breakdown of our LPQLD results on ImageNet-1K, showing how different quantization settings at a fixed base pruning rate affect model performance across IPC values.
\begin{table}[H]
\centering
\begin{minipage}{\appendixwidetablewidth}
\centering
\captionsetup{width=\linewidth}
\caption{
Detailed LPQLD results on ImageNet-1K with quantization applied to pruned labels.
Only configurations with both pruning and quantization.
The base pruning rate is fixed at 10$\times$ for all configurations, with varying Top-K quantization settings.
}
\label{tab:lpqld_in1k}
\newcolumntype{C}{>{\centering\arraybackslash}X}
\begin{tabularx}{\linewidth}{@{}lCCCC@{}}
\toprule
IPC & \makecell{Base\\Pruning Rate} & \makecell{Quantization\\Top-K} & \makecell{Actual\\Compression} & Accuracy (\%) \\ \midrule
\multirow{4}{*}{IPC10} & \multirow{4}{*}{10$\times$} & Top-250 & 20$\times$ & 32.2 \\ 
 &  & Top-100 & 40$\times$ & 29.6 \\
 &  & Top-50 & 79$\times$ & 27.3 \\
 &  & Top-10 & 202$\times$ & 20.0 \\ \midrule
\multirow{4}{*}{IPC20} & \multirow{4}{*}{10$\times$} & Top-250 & 19$\times$ & 43.2 \\
 &  & Top-100 & 45$\times$ & 41.2 \\
 &  & Top-50 & 78$\times$ & 38.6 \\
 &  & Top-10 & 199$\times$ & 30.5 \\ \midrule
\multirow{4}{*}{IPC50} & \multirow{4}{*}{10$\times$} & Top-250 & 19$\times$ & 53.1 \\
 &  & Top-100 & 45$\times$ & 51.3 \\
 &  & Top-50 & 78$\times$ & 49.6 \\
 &  & Top-10 & 199$\times$ & 43.0 \\ \midrule
\multirow{4}{*}{IPC100} & \multirow{4}{*}{10$\times$} & Top-250 & 19$\times$ & 57.5 \\
 &  & Top-100 & 45$\times$ & 56.2 \\
 &  & Top-50 & 78$\times$ & 54.9 \\
 &  & Top-10 & 199$\times$ & 50.1 \\ \midrule
\multirow{4}{*}{IPC200} & \multirow{4}{*}{10$\times$} & Top-250 & 19$\times$ & 60.8 \\
 &  & Top-100 & 45$\times$ & 59.9 \\
 &  & Top-50 & 78$\times$ & 58.8 \\
 &  & Top-10 & 197$\times$ & 55.4 \\ \bottomrule
\end{tabularx}
\end{minipage}
\end{table}

\subsection{ImageNet-21K-P}

\textbf{Dataset.} ImageNet-21K-P~\cite{ridnik2021imagenet} is a pruned version of ImageNet-21K with 10,450 categories and 11,060,223 training images. Images are resized to $224 \times 224$ pixels.

\textbf{Training Configuration.}
Following CDA~\cite{yin2023dataset}, we use ResNet-18 trained for 80 epochs initialized with ImageNet-1K weights~\cite{ridnik2021imagenet}, achieving 38.1\% Top-1 accuracy.
Class-wise BN statistics are computed per Table~\ref{tab:squeeze-in21k}.
Data synthesis and relabeling/validation follow Tables~\ref{tab:recover-in21k} and \ref{tab:validate-in21k}.
We use the same temperature annealing as ImageNet-1K ($\tau$: $20 \to 2$).
CutMix is replaced with CutOut~\cite{devries2017improved}, and label smoothing of 0.2 is applied.
\begin{table}[H]
\centering
\begin{minipage}{\appendixnarrowtablewidth}
\centering
\captionsetup{width=\linewidth}
\caption{Squeezing configurations of Imagenet-21K-P.}
\label{tab:squeeze-in21k}
\begin{tabularx}{\linewidth}{@{}lX@{}}
\toprule
Info          & Detail                                           \\ \midrule
Total Images  & 11,060,223                                       \\
Batch Size    & 1,024                                            \\
BN Updates    & 10,801                                           \\
Source        & \href{https://github.com/Alibaba-MIIL/ImageNet21K}{https://github.com/Alibaba-MIIL/ImageNet21K}       \\ \bottomrule
\end{tabularx}
\end{minipage}
\vspace{-1em}
\end{table}

\begin{table}[H]
\centering
\begin{minipage}{\appendixtablewidth}
\centering
\captionsetup{width=\linewidth}
\caption{Data Synthesis of ImageNet-21K-P.}
\label{tab:recover-in21k}
\begin{tabularx}{\linewidth}{@{}XXX@{}}
\toprule
Config             & Value         & Detail                         \\ \midrule
Iteration          & 2,000         & -                              \\
Optimizer          & Adam          & $\beta_1,\beta_2 = (0.5, 0.9)$ \\
Image LR           & 0.05          & -                              \\
Batch Size         & IPC-dependent & e.g., 20 for IPC20             \\
Initialization     & Random        & -                              \\
BN Loss $(\alpha)$ & 0.25          & -                              \\ \bottomrule
\end{tabularx}
\end{minipage}
\end{table}

\begin{table}[H]
\centering
\begin{minipage}{\appendixtablewidth}
\centering
\captionsetup{width=\linewidth}
\caption{Relabel and Validation of ImageNet-21K-P. $\tau$ denotes temperature.}
\label{tab:validate-in21k}
\begin{tabularx}{\linewidth}{@{}XXX@{}}
\toprule
Config                                            & Value             & Detail                    \\ \midrule
Epochs                                            & 300               & -                         \\
Optimizer                                         & AdamW             & decay $=0.01$             \\
Model LR                                          & 0.002             & -                         \\
Batch Size                                        & 32                & -                         \\
LR Scheduler                                      & CosineAnnealing   & -                         \\
Label Smoothing                                   & 0.2               & -                         \\
EMA Rate                                          & Not Used          & -                         \\
\multirow{2}{*}{Augmentation} & RandomResizedCrop & scale ratio = (0.08, 1.0) \\
                              & CutOut            & -                         \\ \midrule
Initial $\tau$                                    & 20                & -                         \\
Final $\tau$                                      & 2                 & clip at 2                 \\
$\tau$ Scheduler                                  & Step              & $\times 0.7$ every 30 epochs   \\
\bottomrule
\end{tabularx}
\end{minipage}
\end{table}

\textbf{Compression Settings.}
Table~\ref{tab:lpqld_in21k} presents the detailed results for ImageNet-21K-P with various combinations of pruning rates and quantization top-K values.
\begin{table}[H]
\centering
\begin{minipage}{\appendixwidetablewidth}
\centering
\captionsetup{width=\linewidth}
\caption{
Detailed LPQLD results on ImageNet-21K-P with quantization applied to pruned labels.
Each row represents a specific combination of pruning rate and quantization Top-K value.
Best results within each storage tier are highlighted in \textbf{bold}.
The base pruning rate varies (10$\times$, 20$\times$, 40$\times$) with different Top-K quantization settings.
}
\label{tab:lpqld_in21k}
\newcolumntype{C}{>{\centering\arraybackslash}X}
\begin{tabularx}{\linewidth}{@{}lCCCC@{}}
\toprule
IPC & \makecell{Pruning\\Rate} & \makecell{Quantization\\Top-K} & \makecell{Actual\\Compression} & Accuracy (\%) \\ \midrule
\multirow{9}{*}{IPC10} & \multirow{3}{*}{10$\times$} & Top-500 & 100$\times$ & 23.5 \\
 &  & Top-100 & 402$\times$ & 20.5 \\
 &  & Top-50 & 536$\times$ & 18.9 \\ \cmidrule(l){2-5}
 & \multirow{3}{*}{20$\times$} & Top-1000 & 97$\times$ & \textbf{23.9} \\
 &  & Top-200 & 402$\times$ & \textbf{20.9} \\
 &  & Top-100 & 775$\times$ & \textbf{19.2} \\ \cmidrule(l){2-5}
 & \multirow{3}{*}{40$\times$} & Top-2000 & 107$\times$ & 23.3 \\
 &  & Top-400 & 495$\times$ & 19.8 \\
 &  & Top-200 & 835$\times$ & 18.1 \\ \midrule
\multirow{9}{*}{IPC20} & \multirow{3}{*}{10$\times$} & Top-500 & 99$\times$ & 26.1 \\
 &  & Top-100 & 414$\times$ & 23.0 \\
 &  & Top-50 & 559$\times$ & 21.9 \\ \cmidrule(l){2-5}
 & \multirow{3}{*}{20$\times$} & Top-1000 & 92$\times$ & 27.1 \\
 &  & Top-200 & 389$\times$ & 23.9 \\
 &  & Top-100 & 756$\times$ & 22.7 \\ \cmidrule(l){2-5}
 & \multirow{3}{*}{40$\times$} & Top-2000 & 107$\times$ & \textbf{27.7} \\
 &  & Top-400 & 476$\times$ & \textbf{24.1} \\
 &  & Top-200 & 803$\times$ & \textbf{23.3} \\ \bottomrule
\end{tabularx}
\end{minipage}
\end{table}

\section{Additional Information}
\subsection{Class-wise Statistics Storage}
\label{appendix:storage-class-stats}
\begin{table}[H]
\centering
\begin{minipage}{\appendixnarrowtablewidth}
\centering
\captionsetup{width=\linewidth}
\caption{
Storage required for class-wise statistics.
The model is ResNet-18, and storage is measured in MB.
}
\label{tab:storage-class-stats}
\newcolumntype{C}{>{\centering\arraybackslash}X}
\begin{tabularx}{\linewidth}{@{}lCCC@{}}
\toprule
              & ImageNet-1K & ImageNet-21K-P \\ \midrule
Original      & 44.66       & 247.20         \\
+ Class Stats & 81.30       & 445.87         \\ \midrule
Diff.         & 36.64       & 198.67         \\ \bottomrule
\end{tabularx}
\end{minipage}
\end{table}

The additional storage allocation for class-specific statistics is detailed in Table~\ref{tab:storage-class-stats}.
It is observed that this storage requirement escalates with an increase in the number of classes. However, this is a one-time necessity during the recovery phase and becomes redundant once the synthetic data generation is completed.

\subsection{Computing Resources}
\label{appendix:resources}
Experiments are performed on 4 NVIDIA A100 80G GPU cards, with a few experiments performed on RTX 3090 GPU cards.

\section{Limitation and Future Work}
\label{appendix:limitation}
While our method significantly reduces the storage requirements for dataset distillation, we acknowledge several limitations that provide opportunities for future research. First, although we achieve substantial label compression, the training time remains unchanged as the total number of training epochs required for optimal performance is not reduced. Second, the grid search process for finding the optimal student temperature in our calibrated alignment approach introduces additional computational overhead during training, though this cost is negligible compared to the benefits of storage reduction.

Future work could focus on developing methods that simultaneously reduce both storage requirements and training time, perhaps through more efficient knowledge transfer mechanisms or adaptive training schedules. Additionally, extending our approach to other knowledge distillation scenarios beyond dataset distillation (e.g., dataset pruning~\cite{xiao2025rethinkdc, ben2024distilling}) could broaden its applicability. Finally, exploring different temperature schedules tailored to specific architectures or datasets may further improve performance.

\section{Ethics Statement and Broader Impacts}
\label{appendix:ethics}
Our research study focuses on dataset distillation, which aims to preserve data privacy and reduce computing costs by generating small synthetic datasets that have no direct connection to real datasets. 
However, this approach does not usually generate datasets with the same level of accuracy as the full datasets.

In addition, our work in reducing the size of soft labels and enhancing image diversity can have a positive impact on the field by making large-scale dataset distillation more efficient, thereby reducing storage and computational requirements.
This efficiency can facilitate broader access to advanced machine learning techniques, potentially fostering innovation across diverse sectors.

\clearpage

\section{Visualization}
\label{appendix:viz}

In this section, we present visualizations of the datasets used in our experiments.
Due to the different matching objectives, datasets of different IPCs should have distinct images.
Therefore, we provide visualizations for different IPCs.
Figure~\ref{fig:imagenet1k} shows randomly sampled images from ImageNet-1K at various IPC.
Figure~\ref{fig:imagenet21k} provides visualizations of ImageNet-21K-P at IPC10 and IPC20.

\subsection{ImageNet-1K}
\begin{figure}[H]
    \centering
    \begin{minipage}[t]{0.48\columnwidth}
        \centering
        \includegraphics[width=\textwidth]{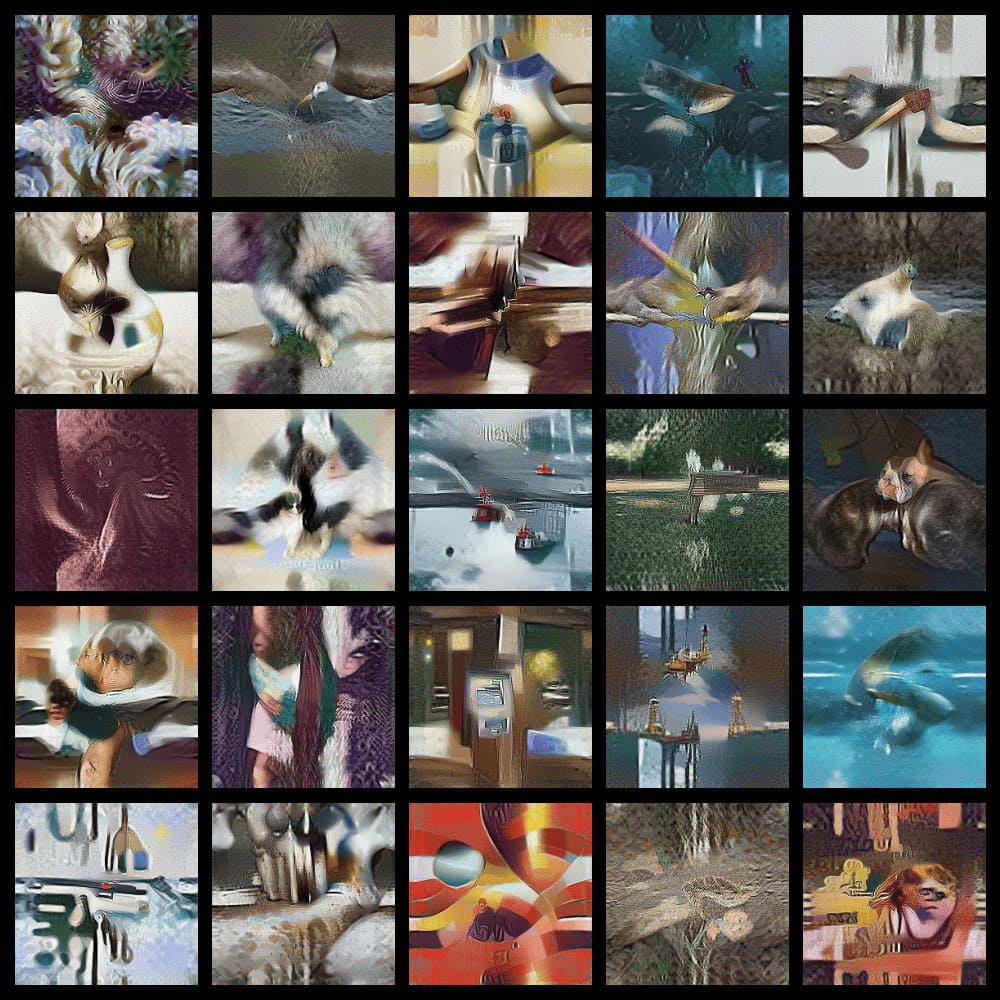}
        \caption{IPC10}
        \label{fig:in1k_ipc10}
    \end{minipage}%
    \hfill
    \begin{minipage}[t]{0.48\columnwidth}
        \centering
        \includegraphics[width=\textwidth]{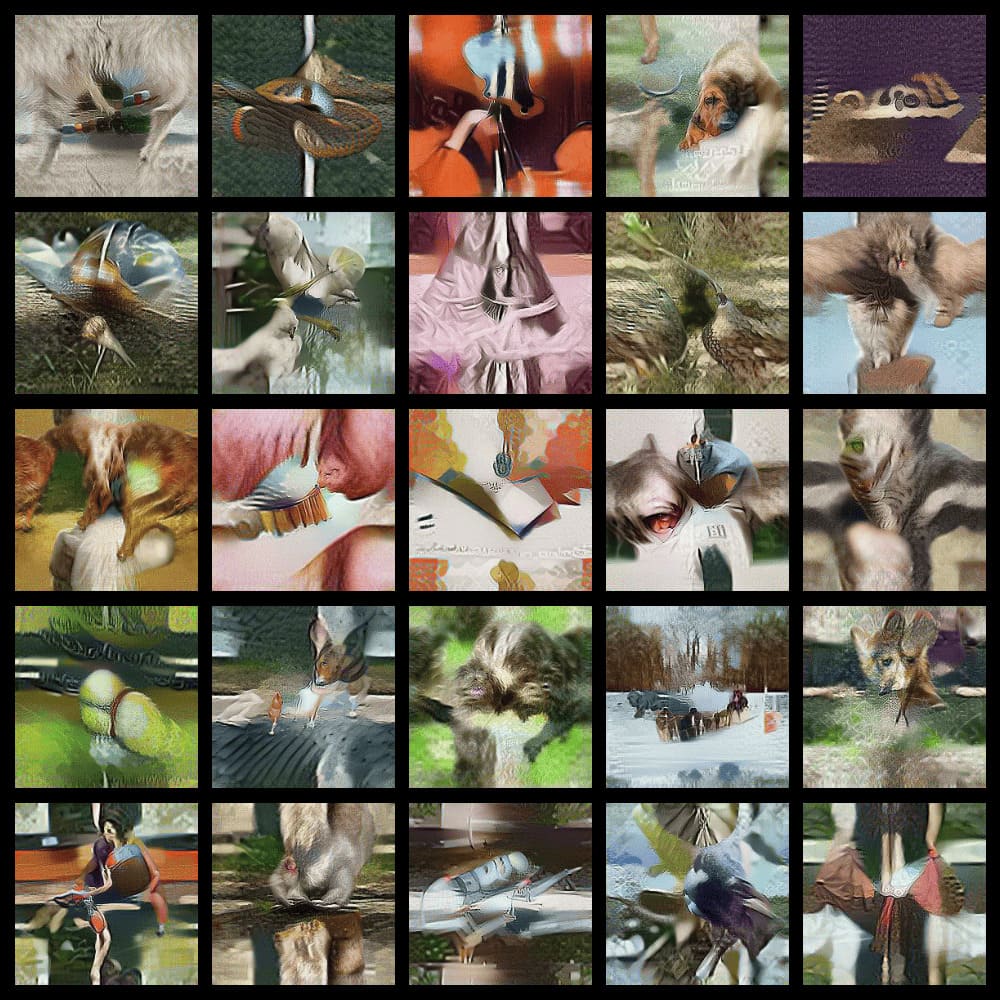}
        \caption{IPC50}
        \label{fig:in1k_ipc50}
    \end{minipage}
    
    \vspace{1cm}
    
    \begin{minipage}[t]{0.48\columnwidth}
        \centering
        \includegraphics[width=\textwidth]{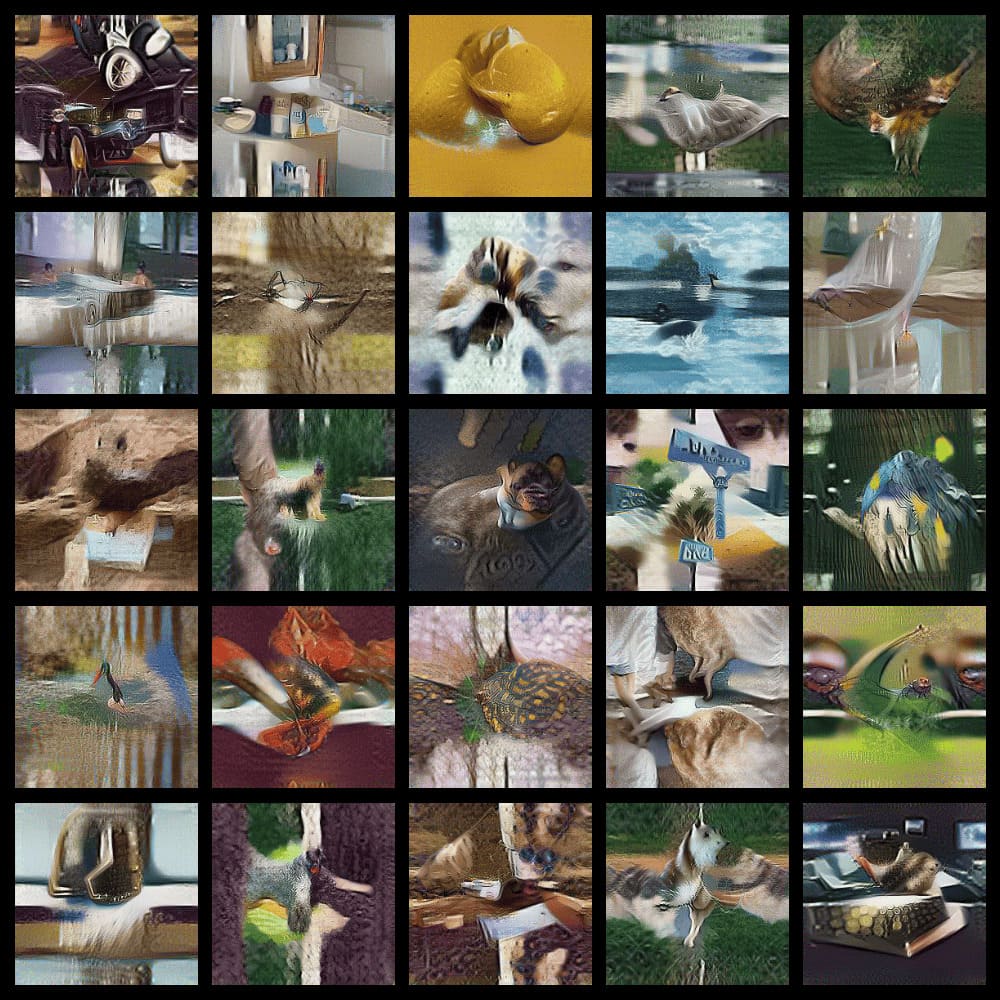}
        \caption{IPC100}
        \label{fig:in1k_ipc100}
    \end{minipage}%
    \hfill
    \begin{minipage}[t]{0.48\columnwidth}
        \centering
        \includegraphics[width=\textwidth]{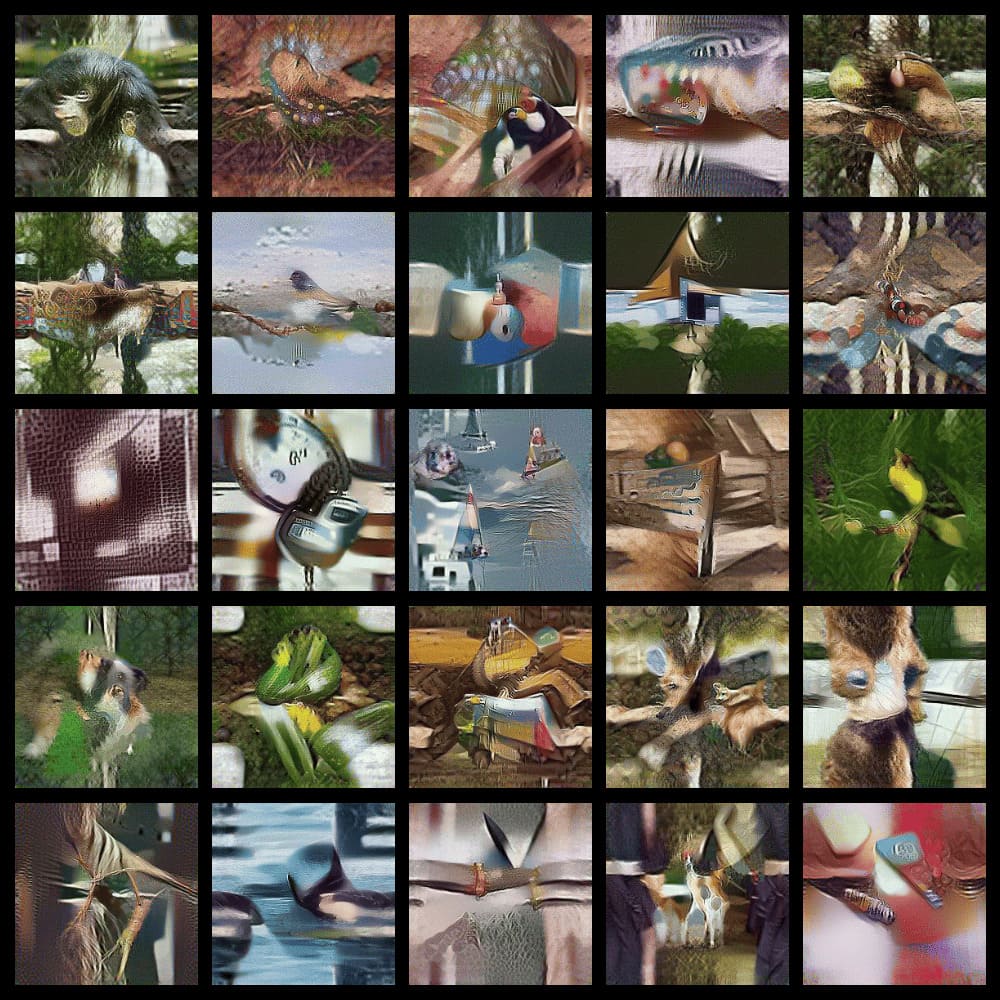}
        \caption{IPC200}
        \label{fig:in1k_ipc200}
    \end{minipage}
    \caption{Visualization of ImageNet-1K. Images are randomly sampled.}
    \label{fig:imagenet1k}
\end{figure}

\subsection{ImageNet-21K-P}
\begin{figure}[H]
    \centering
    \begin{minipage}[t]{0.48\columnwidth}
        \centering
        \includegraphics[width=\textwidth]{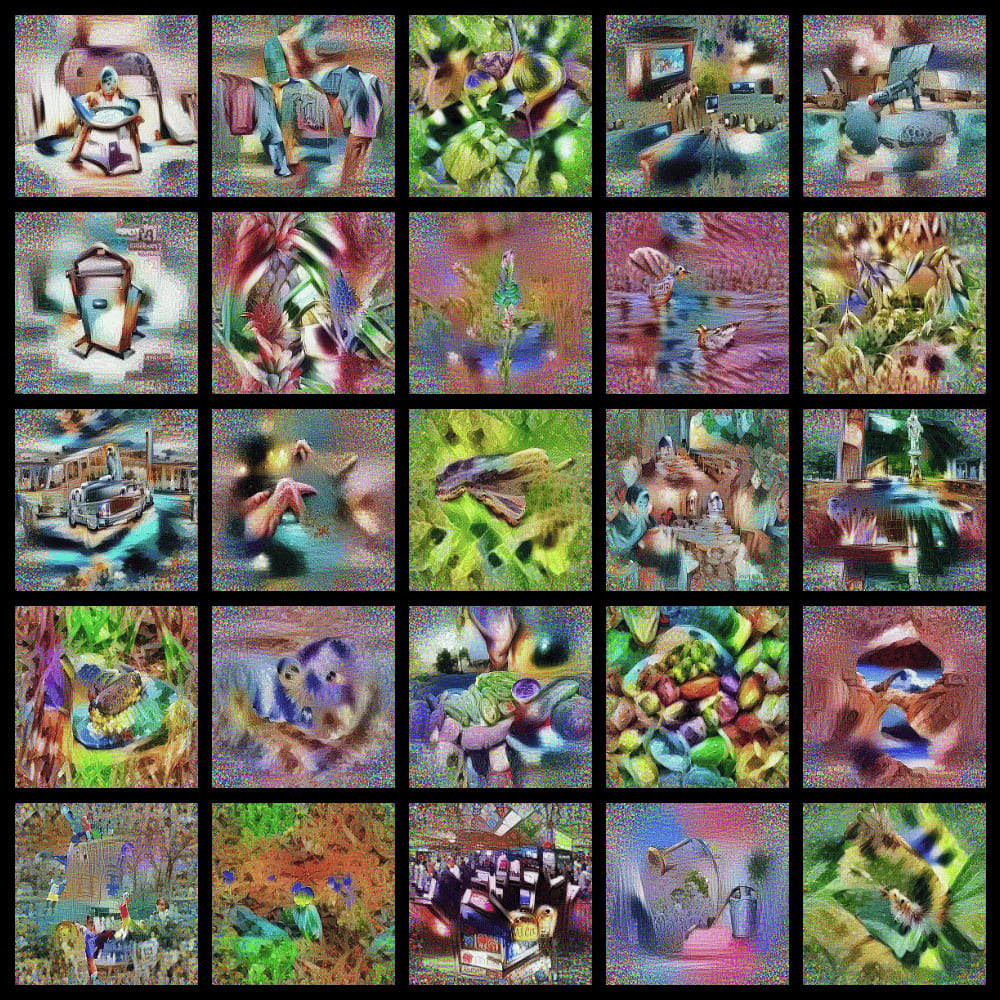}
        \caption{IPC10}
        \label{fig:in21k_ipc10}
    \end{minipage}%
    \hfill
    \begin{minipage}[t]{0.48\columnwidth}
        \centering
        \includegraphics[width=\textwidth]{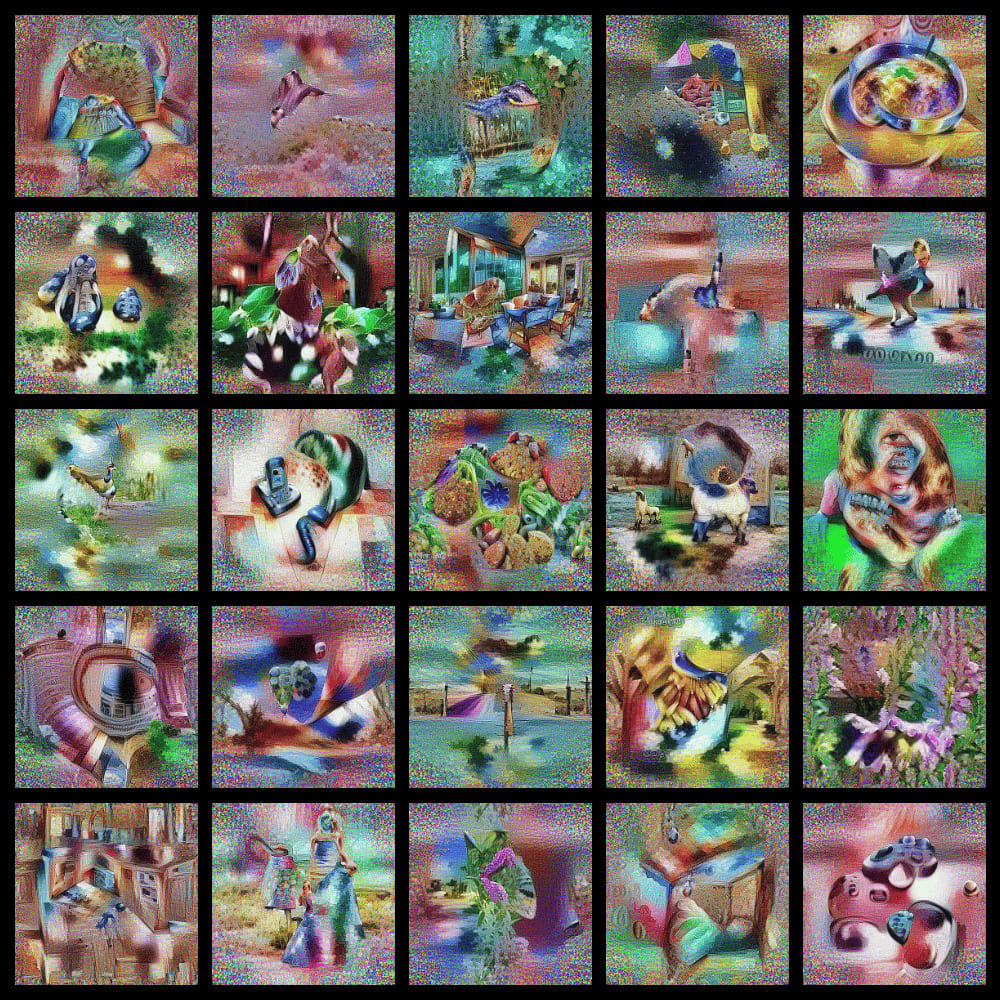}
        \caption{IPC20}
        \label{fig:in21k_ipc20}
    \end{minipage}
    \caption{Visualization of ImageNet-21K-P. Images are randomly sampled.}
    \label{fig:imagenet21k}
\end{figure}

\end{document}